\documentclass[conference]{IEEEtran}
\IEEEoverridecommandlockouts
\usepackage{cite}
\usepackage{amsmath,amssymb,amsfonts}
\usepackage[ruled,linesnumbered]{algorithm2e}
\usepackage{graphicx}
\usepackage{textcomp}
\usepackage{subcaption}
\usepackage{xcolor}

\newcommand{\Ex}[1]{\mathbb{E}[ #1 ]}

\newcommand\bcm{\begin{mat}}
\newcommand\ecm{\end{mat}}

\usepackage{balance}
\usepackage[noend]{algpseudocode}
\usepackage{amsmath,amssymb,amsthm}
\usepackage[shortlabels]{enumitem}
\usepackage{booktabs}
\usepackage{multirow}
\usepackage{setspace}
\usepackage{caption}
\usepackage{graphicx}
\usepackage{float}
\usepackage{hyperref}
\hypersetup{
colorlinks=true,
    citecolor = {blue}
}

\newtheorem{theorem}{Theorem}
\newtheorem{lemma}[theorem]{Lemma}
\newtheorem{proposition}[theorem]{Proposition}
\newtheorem{assumption}{Assumption}
\newtheorem{corollary}[theorem]{Corollary}

 \newtheorem{innercustomthm}{Theorem}
 \newenvironment{customthm}[1]
   {\renewcommand\theinnercustomthm{#1}\innercustomthm}
   {\endinnercustomthm}

 \newtheorem{innercustomcor}{Corollary}
 \newenvironment{customcor}[1]
   {\renewcommand\theinnercustomcor{#1}\innercustomcor}
   {\endinnercustomcor}

\begin{document}

\title{Understanding and Detecting Convergence for Stochastic Gradient Descent with Momentum}

\author{\IEEEauthorblockN
{\bf Jerry Chee,\  Ping Li\vspace{0.08in}}
\IEEEauthorblockA{Cognitive Computing Lab \\
Baidu Research\\
10900 NE 8th St. Bellevue, WA 98004, USA \\
\texttt{jerry9567@gmail.com, ping98@gmail.com}}
}

\maketitle

\begin{abstract}
 Convergence  detection of iterative stochastic optimization methods is of great practical interest.
This paper considers stochastic gradient descent (SGD) with a constant learning rate and momentum. We show that there exists a transient phase in which iterates move towards a region of interest, and a stationary phase in which iterates remain bounded in that region around a minimum point.
We construct a statistical diagnostic test for convergence to the stationary phase using the inner product between successive gradients and demonstrate that the proposed diagnostic works well. We theoretically and empirically characterize how momentum can affect the test statistic of the diagnostic, and how the test statistic captures a relatively sparse signal within the gradients in convergence.
Finally, we demonstrate an application to automatically tune the learning rate by reducing it each time stationarity is detected, and show the procedure is robust to mis-specified initial rates.

\vspace{0.15in}

\end{abstract}

\section{Introduction}

Consider the problem in stochastic optimization
\begin{align}\label{eq:stoch_opt}
\theta_\star &= \arg \min_{\theta \in \Theta} \Ex{ \ell ( \theta, \xi ) }.
\end{align}
The loss $\ell$ is parameterized by $\Theta \subseteq \mathbb{R}^p$, and $\xi$ is a source of randomness like a randomly sampled point.  For example, the quadratic loss is $\ell (\theta, \xi) = (1/2) ( y - x^\top \theta )^2$ with $\xi = (x,y)$. When the data size $N$ and parameter size $p$ are large, classical optimization methods can fail to estimate $\theta_\star$. In such large-scale settings stochastic gradient descent (SGD):
\begin{equation}\label{eq:sgd-vanilla}
\theta_{n+1} = \theta_{n} - \gamma \nabla \ell ( \theta_{n}, \xi_{n+1} )
\end{equation}
is a powerful alternative~\cite{Proc:Bottou_COMPSTAT10, Collect:Bottou_LNCS12, Proc:Zhang_ICML04, Article:Toulis_SAC15}.
$\theta_{n+1}$ is the estimate of $\theta_\star$ at the $(n+1)$-th iteration, and $\gamma > 0$ the learning rate.  $\xi_{n+1}$ represents randomly sampled data used to compute the stochastic gradient.  A mini-batch  can reduce the variance of the stochastic gradients and aid in performance~\cite{Proc:Reddi_NIPS15}.

Momentum or Heavy Ball SGD (SGDM) can offer significant speedups~\cite{Article:Polyak_1964}:
\begin{equation}\label{eq:sgdm}
\theta_{n+1} = \theta_{n} - \gamma \nabla \ell ( \theta_{n}, \xi_{n+1} ) + \beta ( \theta_{n} - \theta_{n-1} )
\end{equation}
where $\beta \in [0,1)$ is the momentum.
The momentum term $\beta ( \theta_{n} - \theta_{n-1} )$ accumulates movements in a common direction.
The performance of stochastic gradient methods is greatly influenced by the learning rate $\gamma$, which can be decreasing (e.g., $\propto 1/n$) or constant.
Decreasing learning rates are commonly used in the literature to attain theoretical convergence guarantees.
However in practice constant learning rates are common due to their ease of tuning and speed of convergence.

Stochastic iterative procedures start from an initial point and then move from a transient phase to a stationary phase~\cite{Article:Murata_1998}.
With a decreasing learning rate, the transient phase can be long, and impractically so if the learning rate is just slightly misspecified~\cite{Article:Nemirovski_SIOPT09, Article:Toulis_AOS17}.
But, the stationary phase is convergence to $\theta_\star$.  With a constant learning rate the transient phase is much shorter and more robust to the learning rate.
The stationary phase is not true convergence but oscillation within a bounded region containing $\theta_\star$. In this study, we develop a statistical convergence diagnostic for SGDM with constant learning rate. Constant learning rate is commonly used in practice, makes the transition from transient to stationary phase
clear, and  it is pointless to keep running the procedure once the stationary phase has been reached.

\subsection{Related work}

The idea that stochastic gradient methods can be separated into a transient and stationary phase (or search and convergence phase) is not new~\cite{Article:Murata_1998}.
However, until recently there has been little work in developing principled statistical methods for convergence detection which can guide empirical practice. Heuristics from optimization theory are commonly used, such as stopping when
$\| \theta_{n} - \theta_{n-1} \|$ is small according to some threshold, or when updates of the loss function have reached machine precision~\cite{Article:Bottou_SIREV18, BooK:Ermoliev_1988}. These methods are more suited for deterministic rather than stochastic procedures as they do not account for the sampling variation in stochastic gradient estimates.
A more statistically motivated approach is to concurrently monitor test error on a hold-out validation set and stop when validation error begins increasing~\cite{Proc:Blum_COLT99,Collect:Bottou_LNCS12}.
But, the validation error is also a stochastic process, and estimating whether it is increasing presents similar, if not greater, challenges to detecting convergence to the stationary~phase.

In stochastic approximation, classical theory of stopping times addresses the detection of stationarity~\cite{Article:Pflug_1990,Proc:Yin_1989}.
One noteworthy method by~\cite{Article:Pflug_1990} forms the basis for our work.
It keeps a running average of the inner product of successive gradients $\nabla \ell ( \theta_{n}, \xi_{n+1})^\top \nabla \ell ( \theta_{n-1}, \xi_{n} )$.
At a high level, in the transient phase the stochastic gradients generally point in the same direction, resulting in a positive inner product.
In the stationary phase the stochastic gradients roughly point in different directions due to their oscillation in a bounded region containing $\theta_\star$, resulting in a negative inner product.
Accelerated methods in stochastic approximation share the underlying intuition that a negative inner product of successive gradients indicates convergence~\cite{Article:Delyon_SIOPT93, Article:Kesten_AOMS58, Proc:Roux_NIPS12}.

There has been recent interest in principled convergence detection for stochastic gradient methods, and automated step decay learning rates.
Work by~\cite{Proc:Chee_AISTATS18} developed a principled convergence diagnostic for  SGD in Eq.~(\ref{eq:sgd-vanilla}) based on Pflug's procedure. We generalize Pflug's procedure to the momentum setting, which introduces challenges to the theoretical justification and practicality of the convergence diagnostic.~\cite{Report:Shirish_arXiv17} developed a procedure to automatically switch from Adam~\cite{Proc:Kingma_ICLR15} to SGD.~\cite{matteo2019JSM} propose a modified splitting procedure~\cite{Report:Su_arXiv18} to detect the stationary phase for SGD and implement a robust learning rate schedule.~\cite{Proc:Lang_NeurIPS19} use a Markov chain $t$-test and a stationarity condition by~\cite{Proc:Yaida_ICLR19} to automatically reduce the learning rate for SGDM.~\cite{Proc:Ge_NeurIPS19} analyze the step decay learning rate schedule in least squares regression and show its optimality over polynomially decaying rates.

\subsection{Our contributions}

Section~\ref{sec:sgd-vanilla} presents a statistical convergence diagnostic for SGDM (stochastic gradient descent with momentum) and explains the significance of the challenges introduced by momentum. In Section~\ref{sec:momentum-stationarity} we provide theoretical and empirical support for the design choices of the convergence diagnostic, and demonstrate the effect of momentum on the test statistic of the diagnostic.
We investigate in Section~\ref{sec:distribution_IP}
 what drives the test statistic by analyzing the distribution of the inner product of successive gradients in the stationary phase. In Section~\ref{sec:synth-data-experiments} we provide empirical type I and type II error rates on simulated data experiments. Section~\ref{sec:autoLR} presents an application of the convergence diagnostic to an automatically tuned learning rate schedule, with experiments on benchmark datasets.

\section{Convergence diagnostic}\label{sec:sgd-vanilla}

Our convergence diagnostic aims to detect the transition from the transient to the stationary phase. We first present theory which supports the existence of these two phases for SGDM.
The expected difference in loss to the minimum has bias terms due to initial conditions, and a variance term due to noise in the stochastic gradients.

\begin{theorem}[\cite{Report:Yang_arXiv16}]
\label{thm:mom_convg_analysis}
If the expected loss $f(\theta) = \Ex{\ell(\theta,\xi)}$ is convex, under additional assumptions of the loss,  there are positive constants $Q_\beta, R_\beta, S_\beta$ such that for every $n$, we have
\begin{align*}
\mathbb{E} [ f ( \hat{\theta}_{n} ) - f ( \theta_\star ) ] &\leq \frac{Q_\beta}{n+1} ( f ( \theta_0 ) - f ( \theta_\star ) ) \\
&\quad + \frac{R_\beta}{\gamma ( n+1 )} \| \theta_0 - \theta_\star \|^2 + \gamma S_\beta. \hspace{0.2in} \Box
\end{align*}
\end{theorem}

\emph{Remarks.}
$\hat{\theta}_n = \sum_{t=0}^n \theta_t / (n+1)$, $Q_\beta = \beta / ( 1 - \beta ), R_\beta = ( 1 - \beta ) / 2$, and $S_\beta = ( G^2 + \delta^2 ) / 2 ( 1 - \beta )$ where $G$ is a bound on the gradients and $\delta^2$ a bound on the variance of the stochastic gradients.
For large enough $n$ the bias contributions from the transient phase are negligible, and thus bounded $\Ex{ f(\hat{\theta}_n) - f(\theta_\star) }$ indicates a bounded $\Ex{ f(\theta_n) - f(\theta_\star) }$~in~the~stationary~phase.\\

Theorem~\ref{thm:mom_convg_analysis} suggests that constant rate SGDM moves quickly through the transient phase discounting initial conditions $f ( \theta_0 ) - f ( \theta_\star )$ and $\| \theta_0 - \theta_\star \|^2$, and then enters the stationary phase where the distance from $\theta_\star$ is bounded $\propto O ( \gamma )$.
We observe a widely noted  trade-off  for stochastic gradient methods:  a larger learning rate speeds up the transient phase by discounting bias from initial conditions at a higher rate, but increases the radius of the stationary region~\cite{Proc:Bach_NIPS11, Article:Needell_MP16}.

We cite~\cite{Report:Yang_arXiv16} in Theorem~\ref{thm:mom_convg_analysis} because their convergence rate consisting of a reducible and irreducible term with respect to the number of updates, and best matches our empirical observations. Though~\cite{Report:Loizou_arXiv17} show a linear convergence rate with constant stepsize, their restriction on the momentum ($\beta$) makes their convergence rate difficult to realize in practice. For example, using the formula in~\cite{Report:Loizou_arXiv17} with min and max eigenvalues $=0.5$ and stepsize $=0.1$, $\beta < 0.2$, which~is~very~restrictive.

While convergence analyses such as Theorem~\ref{thm:mom_convg_analysis} offer valuable theoretical insight, they provide limited practical guidance. One could try to declare convergence when the bias due to initial conditions has been discounted to $1\%$ of the variance,
choosing $n$ for
$\left[ \frac{Q_\beta}{n+1} ( f ( \theta_0 ) - f ( \theta_\star ) ) + \frac{R_\beta}{\gamma ( n+1 )} \| \theta_0 - \theta_\star \|^2 \right] = 0.01  \gamma S_\beta$.
But estimating $f ( \theta_0 ) - f ( \theta_\star)$, $ \| \theta_0 - \theta_\star \|^2$, $G^2$, and $\delta^2$ is difficult.
We provide an alternative, by developing a practical statistical diagnostic test to estimate the phase transition and detect convergence of SGDM in a much~simpler~way.

\subsection{Modified Pflug diagnostic}

\begin{algorithm}[h]
\SetKwProg{Fn}{function}{:}{}
\SetKwFunction{fMS}{momentum\_switch}
\SetKwProg{fn}{function}{:}{}
\SetKwInOut{Input}{input}
\Input{Initial point $\theta_0$,
data $\{ (x_1,y_1)$, $(x_2,y_2)$, $\dots \}$,
$\gamma > 0$,
$\beta \in [0,1)$,
final momentum $\beta' \in [0, \beta)$, heuristic convergence $h$,
threshold $T>0$, checking period $c>0$,
$\tt{burnin} > 0$.
}
$S \gets 0$; \
$\alpha \gets 0$ \\
Sample $\xi_1 \gets (x_1, y_1)$ \\
$\theta_1 \gets \theta_0 - \gamma \nabla \ell (\theta_0, \xi_1)$ \\
\For {$n \in \{2,3,\dots\}$} {
Sample $\xi_n = (x_n, y_n)$ \\
$\theta_n \gets \theta_{n-1} - \gamma \nabla \ell (\theta_{n-1}, \xi_n) + \beta (\theta_{n-1} - \theta_{n-2})$ \\
$\alpha$, $\beta \gets$ \fMS{$n$, $\alpha$, $\beta$, $\nabla \ell (\theta_0, \xi_1)$, $\dots$, $\nabla \ell (\theta_{n-1}, \xi_n)$} \\
\If {$\alpha > 0$ and $n > \alpha + \tt{burnin}$} {
$S \gets S + \nabla \ell (\theta_{n-1}, \xi_n)^\top \nabla \ell (\theta_{n-2}, \xi_{n-1})$
\label{line:test_stat} \\
\If {$S < 0$ and $n$\hspace{-3pt}$\mod c == 0$} {
\KwRet $\theta_n$
}
}
}
\fn{\fMS{$n$, $\alpha$, $\beta$, $\nabla \ell (\theta_1, \xi_1)$, $\dots$, $\nabla \ell (\theta_{n-1}, \xi_n)$}}{
\label{line:mom_start}
\If{$h(\nabla \ell( \theta_0,\xi_1), \dots \nabla \ell (\theta_{n-1},\xi_n)) < T$ and $n$\hspace{-3pt}$\mod c == 0$ and $\alpha == 0$} {
\label{line:heur_convg}
$\alpha \gets n$ \\
$\beta \gets \beta'$
}
\KwRet $\alpha$, $\beta$
\label{line:mom_end}
}
\caption{Convergence diagnostic for SGDM.}
\label{alg:diagnostic}
\end{algorithm}

We present a convergence diagnostic for SGDM in Alg.~\ref{alg:diagnostic}.
We draw upon Pflug's procedure in stochastic approximation~\cite{Article:Pflug_1990}, and generalize the procedure in~\cite{Proc:Chee_AISTATS18} to momentum. In the transient phase SGDM
moves quickly towards $\theta_\star$ by discarding initial conditions, and so gradients likely point in the same direction.
This implies on average a positive inner product.
In the stationary phase SGDM oscillates in a region around $\theta_\star$, indicating the gradients point in different directions.
This implies on average a negative inner product.
Thus a change in sign from positive to negative inner products is a good indicator that convergence has been reached.

Momentum introduces two significant challenges to the development of a convergence diagnostic.
First, the test statistic of the diagnostic needs to be constructed.
Pflug's procedure takes an inner product between successive gradients, which can be rewritten as $\frac{1}{\gamma^2} (\theta_{n+1} - \theta_n)^\top(\theta_n - \theta_{n-1})$ since by Eq.~(\ref{eq:sgd-vanilla}), $\theta_{n} - \theta_{n+1} = \gamma \nabla \ell (\theta_{n}, \xi_{n+1})$.
But with momentum the updates become $(\theta_n - \theta_{n+1}) = \gamma \nabla \ell(\theta_{n}, \xi_{n+1}) - \beta(\theta_n - \theta_{n-1})$, by Eq.~(\ref{eq:sgdm}).
It is unclear what linear combination of the gradient $\nabla \ell(\theta_n, \xi_{n+1})$ and momentum term $\beta(\theta_n - \theta_{n-1})$ should be included in the inner product. Second, regardless of what linear combination is chosen, with high momentum the test statistic can have positive expectation in the stationary phase.
This is a serious issue because a negative expectation allows the use of a threshold of zero.
A zero threshold is attractive because it is  independent of data distribution and loss function. If the inner products are expected positive in the stationary phase, the threshold to declare convergence now depends on these factors and would require an additional estimation problem~to~set.

The convergence diagnostic for SGDM in Alg.~\ref{alg:diagnostic} effectively resolves these issues from momentum. It is defined by a random variable $S$ (line~\ref{line:test_stat}) which keeps the running average of the inner product $\nabla \ell (\theta_n, \xi_{n+1})^\top \nabla \ell (\theta_{n-1}, \xi_n)$ of the gradient at successive iterates.  In Section~\ref{sec:momentum-stationarity} we provide theory which shows that this choice of inner product is more easily able to attain the desired negative expectation, and is more robust to the momentum $\beta$. To remedy the high momentum issue, $\beta$ is automatically reduced (lines~\ref{line:mom_start}-\ref{line:mom_end}) at a point determined by an optimization heuristic to be close to the stationary phase. This does not greatly effect the convergence rate as momentum is most useful in the transient phase. We can think of the momentum reduction point as a noisy estimate of convergence, followed by a more accurate estimate with the convergence diagnostic.
A gradient norm based heuristic function was used (line~\ref{line:heur_convg}).
Often convergence heuristics are calculated in an online manner, so storage is not an issue.

\section{The difficulty with momentum}\label{sec:momentum-stationarity}

We now provide theoretical and empirical justification for the design choices in Alg.~\ref{alg:diagnostic} regarding the challenges due to momentum.
The overall goal is a test statistic with negative expectation to enable the use of a practical zero threshold. We first present two theorems to address the choice of what combination of gradient $\nabla \ell (\theta_n, \xi_{n+1})$ and momentum $\beta (\theta_n - \theta_{n-1})$ should be used to construct the test statistic of the convergence diagnostic. Then we provide a corollary, and theoretical and empirical results in quadratic loss to study the effects of high momentum on the expectation of the chosen test statistic.
We now list the assumptions.
\begin{assumption}
\label{assump:strcvx}
The expected loss $f(\theta) = \Ex{ \ell(\theta, \xi) }$ is strongly convex with constant $c$.
\end{assumption}

\begin{assumption}
\label{assump:Lsmooth}
The expected loss $f(\theta) = \Ex{ \ell(\theta, \xi) }$ is Lipschitz-smooth with constant $L$.
\end{assumption}

\begin{assumption}
\label{assump:Fbound}
Theorem~\ref{thm:mom_convg_analysis}~\cite{Report:Yang_arXiv16} holds s.t. $\Ex{ f(\theta_n) - f(\theta_\star) } \leq \gamma M$ for some $M > 0$ and large enough $n$.
\end{assumption}

\begin{assumption}
\label{assump:min_noise}
$\exists \ \sigma_0^2 > 0$ s.t. $\Ex{ \| \nabla \ell (\theta, \xi) \|^2 } > \sigma_0^2$.
\end{assumption}

\begin{assumption}
\label{assump:scaling}
$\exists \ K > 1$ s.t. $\Ex{ (\theta_n - \theta_{n-1})^\top (\theta_{n-1} - \theta_{n-2}) } \geq - K \Ex{ \| \theta_n - \theta_{n-1} \|^2 }$ for large enough $n$.
\end{assumption}

\emph{Remarks.}
Assumptions~\ref{assump:strcvx} and~\ref{assump:Lsmooth} are standard in analysis of stochastic gradient methods~\cite{Bach_NIPS13,Proc:Bach_NIPS11}.
Assumption~\ref{assump:Fbound} requires $\| \nabla f(\theta) \| \leq G$ and $\Ex{ \| \nabla \ell(\theta, \xi) - \nabla f(\theta) \|^2 } \leq \delta^2$~\cite{Report:Yang_arXiv16}.
Assumption~\ref{assump:min_noise} posits a minimum amount of noise present in the stochastic gradients.
In Assumption~\ref{assump:scaling} it is reasonable to assume $K$ is not too large because $\| \theta_n - \theta_{n-1} \|^2 \approx \| \theta_{n-1} - \theta_{n-2} \|^2$ in the stationary phase.

\subsection{Constructing a test statistic}

We select as the test statistic a running mean of
\begin{equation}
\label{eq:ip_loss}
\nabla \ell ( \theta_{n}, \xi_{n+1})^\top \nabla \ell ( \theta_{n-1}, \xi_{n} ).
\end{equation}
In the following Theorems~\ref{thm:ip_exp_bd} and~\ref{thm:ip_opt} we derive upper bounds on the expected values of different inner products,
and use these results to choose the test statistic.

\begin{theorem}
\label{thm:ip_exp_bd}
Suppose Assumptions~\ref{assump:strcvx},~\ref{assump:Fbound}, and~\ref{assump:min_noise} hold.
Define $A_\beta =  1 / (1 + 2 \beta K + \beta^2)$.
The test statistic in Eq.~(\ref{eq:ip_loss}) for the  diagnostic in Alg.~\ref{alg:diagnostic} for SGDM in Eq.~(\ref{eq:sgdm}) is bounded
\begin{equation*}
\mathbb{E} [ \nabla \ell ( \theta_n, \xi_{n+1} )^\top \nabla \ell ( \theta_{n-1}, \xi_n ) ]
\leq \big( 1 + \beta \big) \left[ M - \frac{c}{2} \gamma \sigma_0^2 A_\beta \right].\hspace{0.1in}\Box
\end{equation*}
\end{theorem}

\vspace{0.1in}

\begin{theorem}
\label{thm:ip_opt}
Suppose that Assumptions~\ref{assump:strcvx},~\ref{assump:Fbound}, and~\ref{assump:min_noise} hold.
Define $\nabla \ell_{n+1} = \nabla \ell (\theta_n, \xi_{n+1})$ and $\Delta_n = (\theta_n - \theta_{n-1})$.
The expectation of the alternative test statistic is bounded
\begin{align}\label{eq:alt_test_stat}
&\Ex{ ( \nabla \ell _{n+1} + \beta \Delta_n )^\top ( \nabla \ell_n + \beta \Delta_{n-1} ) } \\\notag
<& \big( \frac{1}{\gamma} + \frac{\beta}{\gamma} + 2 \beta + \beta^2 \big) \left[ \gamma M - \frac{c}{2} \gamma^2 \sigma_0^2 A_\beta \right]
+ \beta^3 \gamma M.\hspace{0.1in}\Box
\end{align}
\end{theorem}
\emph{Remarks.}
By Theorem~\ref{thm:ip_opt} the alternative test statistic in Eq.~(\ref{eq:alt_test_stat}) is less likely to achieve a negative expectation in stationarity, and thus be able to use a zero threshold.
This is because of the additional $\beta^3 \gamma M > 0$ term.
A choice of another constant $t$ other than $\beta$ in the linear combination  in Eq.~(\ref{eq:alt_test_stat}) would not change this conclusion.  The sign of the bound is not controlled by $t$, and the last term would be $t^2 \beta \gamma M > 0$ and not change sign.
The convergence threshold for a good test statistic should not depend too much on momentum, but the alternative test statistic has increased dependence due to the $2\beta, \beta^2$ terms. Also, note that the test statistic still has dependence on the momentum. Eq.~(\ref{eq:ip_loss}) can be rewritten as
\begin{align}
\label{eq:ip_breakdown}
\frac{\beta}{\gamma} [ \nabla \ell_{n+1}^\top (\theta_{n-1} - \theta_{n-2} ) ]
- \frac{1}{\gamma} [ \nabla \ell_{n+1}^\top ( \theta_n - \theta_{n-1} )].
\end{align}

\subsection{Effect of high momentum}

The test statistic in Eq.~(\ref{eq:ip_loss}) is chosen to best ensure a negative expectation in the stationary phase.
The next corollary guarantees this negative expectation under certain conditions~of~the~learning~rate.
\begin{corollary}
\label{cor:ip_exp_neg}
Consider SGDM in Eq.~(\ref{eq:sgdm}). If
the learning rate satisfies $\gamma > 2 M / c \sigma_0^2 A_\beta$, then,
\begin{equation*}
\mathbb{E} [ \nabla \ell ( \theta_n, \xi_{n+1} )^\top \nabla \ell ( \theta_{n-1}, \xi_n ) ]
< 0
\end{equation*}
as $n \rightarrow \infty$.
Thus the convergence diagnostic activates almost surely.\hfill$\Box$
\end{corollary}
\emph{Remarks.}
Again, $A_\beta =  1 / (1 + 2 \beta K + \beta^2)$.
$A_\beta$ is a monotonically decreasing function where $A_{\beta | \beta=0} = 1$ and $A_{\beta | \beta=1} = 1 / (2 + 2K)$.
Higher $\beta$ reduces $A_\beta$, restricting the condition $\gamma > 2M / c \sigma_0^2 A_\beta$.\\

Corollary~\ref{cor:ip_exp_neg} suggests that too large momentum may make the learning rate condition too prohibitive, invalidating the negative expectation in practice.

\subsection{Quadratic loss model}\label{sec:quadratic}

Now that we have constructed our test statistic, we want to gain insight into the convergence diagnostic and the effect of high momentum.
We do this by considering quadratic loss $\ell ( \theta, y, x ) = \frac{1}{2} ( y - x^\top \theta )^2$ with gradient $\nabla \ell ( \theta, y, x ) = - ( y - x^\top \theta) x$.
Let $y = x^\top \theta_\star + \epsilon$, where $\epsilon$ are zero mean random variables $\Ex{ \epsilon | x } = 0$. $\theta_0 = \theta_\star$; the procedure has started in the stationary region.
The first three iterates are:
\begin{align*}
\theta_1 &= \theta_\star + \gamma ( y_1 - x_1^\top \theta_\star ) x_1 \\
\theta_2 &= \theta_1 + \gamma ( y_2 - x_2^\top \theta_1 ) x_2 + \beta ( \theta_1 - \theta_\star ) \\
\theta_3 &= \theta_2 + \gamma ( y_3 - x_3^\top \theta_2 ) x_3 + \beta ( \theta_2 - \theta_1 )
\end{align*}
Three steps are taken in order for the momentum to effect both terms of the inner product.
The expected value of the test statistic at $\theta_3$ is:
\begin{align}\notag
&\Ex{ \nabla \ell ( \theta_2, y_3, x_3 )^\top \nabla \ell ( \theta_1, y_2, x_2 ) }  \\
=& - \gamma \Ex{ \epsilon_2^2 } \Ex{ ( x_3^\top x_2 )^2 } - \gamma^3 \Ex { \epsilon_1^2 } \Ex { ( x_2^\top x_1 )^2 ( x_3^\top x_2 )^2 } \nonumber \\\label{eq:quad_start_star}
& + \gamma^2 ( 1 + \beta ) \Ex{ \epsilon_1^2 } \Ex{ ( x_2^\top x_1 ) ( x_3^\top x_1 ) ( x_3^\top x_2 ) }
\end{align}
From Eq.~(\ref{eq:quad_start_star}) in the $\gamma^2 (1+\beta)$ term we see that momentum contributes positively to the test statistic, and if it is too large the expectation is positive.  $\Ex{ ( x_2^\top x_1 ) ( x_3^\top x_1 ) ( x_3^\top x_2 ) } = tr( \Ex{(x_1 x_1^\top) (x_2 x_2^\top) (x_3 x_3^\top) } ) > 0$ by application of trace and $x_1 x_1^\top$ is positive definite.

The results are generalized in the following theorem.
\begin{theorem}
\label{thm:quad_diag}
Suppose that the loss is quadratic, $\ell ( \theta ) = 1/2 ( y - x^\top \theta )^2$.
Let $x_n$ and $x_{n+1}$ be two iid vectors from the distribution of $x$.
Let $A = \Ex{ ( x_n x_{n+1}^\top ) ( x_n^\top x_{n+1} ) }$, $B = \Ex{ ( x_{n} x_n^\top ) ( x_n^\top x_{n+1} )^2 }$, $\sigma_{quad}^2 = \Ex { \epsilon_n^2 }$, $d^2 = \Ex{ ( x_n^\top x_{n+1} )^2 }$.
Then for $\gamma > 0$, we have
\begin{align}\notag
&\Ex{ \nabla \ell ( \theta_n, \xi_{n+1} )^\top \nabla \ell ( \theta_{n-1}, \xi_n ) |  \theta_{n-1}, \theta_{n-2} } \\\notag
=& ( \theta_{n-1} - \theta_\star )^\top ( A - \gamma B) ( \theta_{n-1} - \theta_\star )
- \gamma \sigma_{quad}^2 d^2 \\\notag
&\quad\quad + ( \theta_{n-1} - \theta_\star )^\top ( \beta A ) ( \theta_{n-1} - \theta_{n-2} ).\hspace{0.4in} \Box
\end{align}
\end{theorem}

\emph{Remarks.}
The momentum term $\beta A$ only becomes significant in the stationary phase when $\| \theta_{n-1} - \theta_\star \|^2 \approx \| \theta_{n-1} - \theta_{n-2} \|^2$.
It makes an expected positive contribution as $\theta_{n-1}$ and $\theta_{n-2}$ are more likely to be on opposite sides of $\theta_\star$ in the stationary phase. Otherwise, progress towards $\theta_\star$ is still being made in the transient phase.

In the transient phase the bias dominates, resulting in an expected positive contribution from $(A-\gamma B)$ to the test statistic.
In the stationary phase the variance dominates, resulting in an expected negative contribution from $- \gamma \sigma_{quad}^2 d^2$.
Theorem~\ref{thm:quad_diag} supports that in stationarity momentum $\beta$ contributes positively to the test statistic, and $\beta$ that is too high makes the expected value of the test statistic positive.

We empirically validate the effect of high momentum on the test statistic for quadratic loss. We sample $1000$ data points $x \sim N ( 0, I_{20} )$, set $y = x^\top \theta_\star + \epsilon$ with $\epsilon \sim N ( 0, 1 )$, and $\theta_{\star, i} = (-1)^{i} \ 2 \exp( -0.7 i )$ for $i = 1, \dots, 20$.
SGDM is run with batch size $25$ for $50$ epochs.
The stationary phase is marked when the MSE with respect to $\theta_\star$ has flattened out. Table~\ref{tab:high_momentum_ip} reports the mean test statistic in stationarity across 25 independent runs of SGDM. We set $\beta=0.2$ and $\beta=0.9$ to contrast low and high momentum. Both settings attain equivalent MSE. The  results in Table~\ref{tab:high_momentum_ip} support the observations from Corollary~\ref{cor:ip_exp_neg} and Theorem~\ref{thm:quad_diag}: the expectation of the test statistic becomes positive with too~large~momentum.

\begin{table}[h]
\caption{Mean test statistic in stationarity across 25 independent runs with $\gamma = 10^{-2}$.
Low $\beta$ setting has $\beta=0.2$.
High $\beta$ setting has $\beta=0.8$. \vspace{-0.1in}
}
\label{tab:high_momentum_ip}
\begin{center}
\begin{tabular}{|c| c c|}
\hline
 & Low $\beta$ & High $\beta$ \\
\hline
Test Statistic in Stationarity & -6.71 & 2.77\\
\hline
\end{tabular}
\end{center}
\end{table}

\section{Distribution  of inner products}\label{sec:distribution_IP}

The convergence diagnostic crucially relies upon the negative expectation of its test statistic.
An important question emerges: \textbf{what drives the expectation of the inner product negative?
What is its relation to momentum?}
At a high level, there is an oscillation in the stationary phase driven by the dominating variance of the stochastic gradients.
This oscillation interacts with the curvature of the loss function around $\theta_\star$, driving the expectation of inner products negative.
We propose a more refined view which also helps explain the observed sensitivity of the expectation to high momentum.

\begin{proposition}\label{prop:noise_key_iter}

In the stationary phase there are a small number of key iterates which drive the expected inner product in Eq.~(\ref{eq:ip_loss}) negative.
Consider the decomposition of the stochastic~gradient
\begin{equation*}
\nabla \ell ( \theta_n, \xi_{n+1} ) = \Ex{ \nabla \ell ( \theta_n, \xi_{n+1} )} + \sigma^2
\end{equation*}
into its true gradient and noise.  A majority of inner products\\ $\nabla \ell ( \theta_n, \xi_{n+1} ) ^\top \nabla \ell ( \theta_{n-1}, \xi_{n} )$ are mean zero due to the dominance of the noise term $\sigma^2$.  \textbf{The expectation is negative due to a relatively sparse number of inner products which have high magnitude and negative sign.}
In these few key inner products the true gradient dominates because of an interaction with the loss curvature.
This behavior depends on a sufficiently small momentum.
\hfill$\Box$
\end{proposition}

\emph{Remarks.}
There are some iterates $\theta_n$ in the stationary phase which get relatively farther from $\theta_\star$.
This would require a relatively large gradient $\nabla \ell (\theta_{n-1}, \xi_n)$ pointing generally away from $\theta_\star$.
The following iteration would result in an also large gradient $\nabla \ell (\theta_{n}, \xi_{n+1})$ pointing generally back towards $\theta_\star$ due to the curvature of the loss.

\begin{figure}[h]
\mbox{\hspace{-0.1in}
  \includegraphics[width=1.8in]{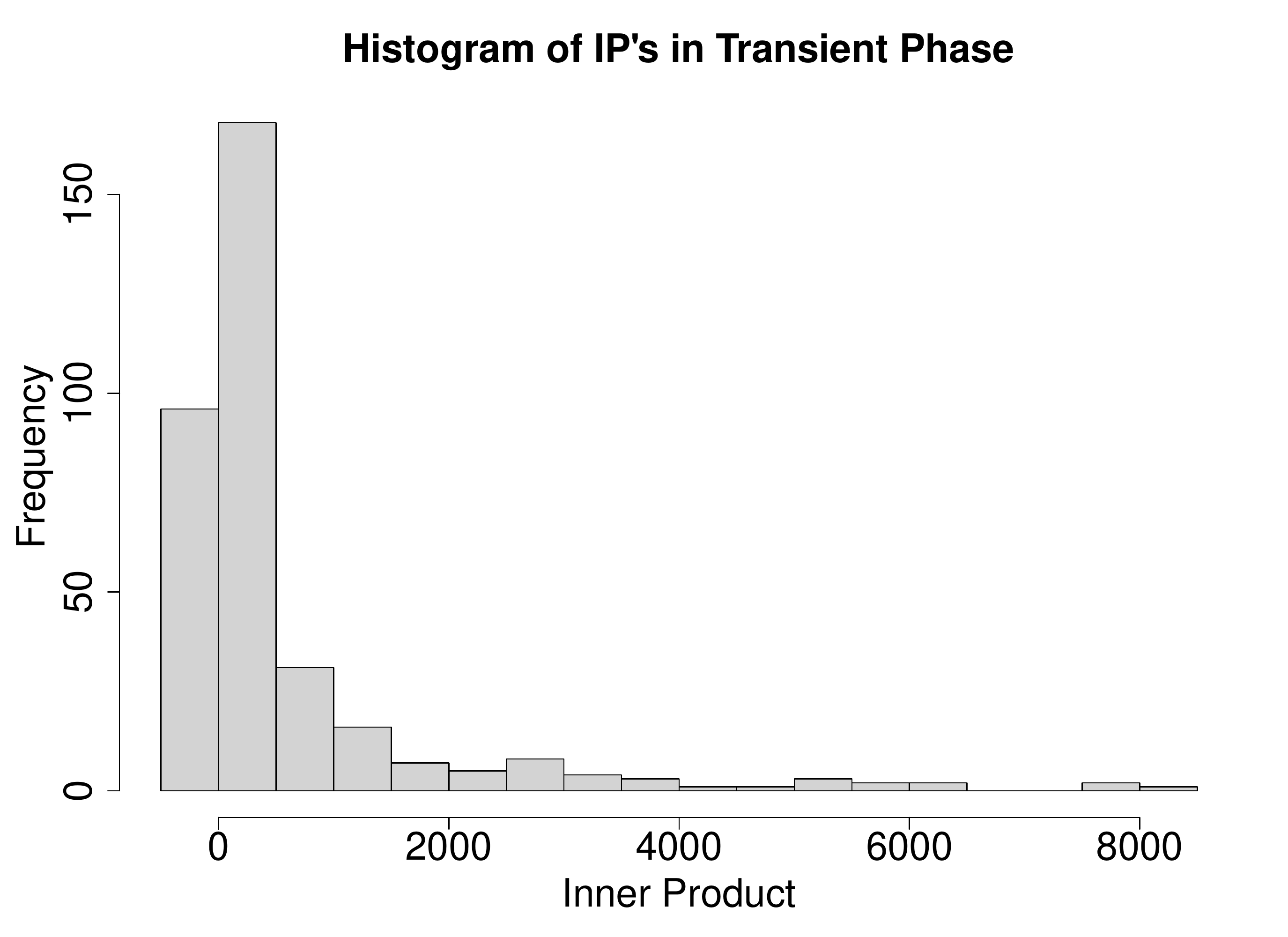}\hspace{-0.12in}
  \includegraphics[width=1.8in]{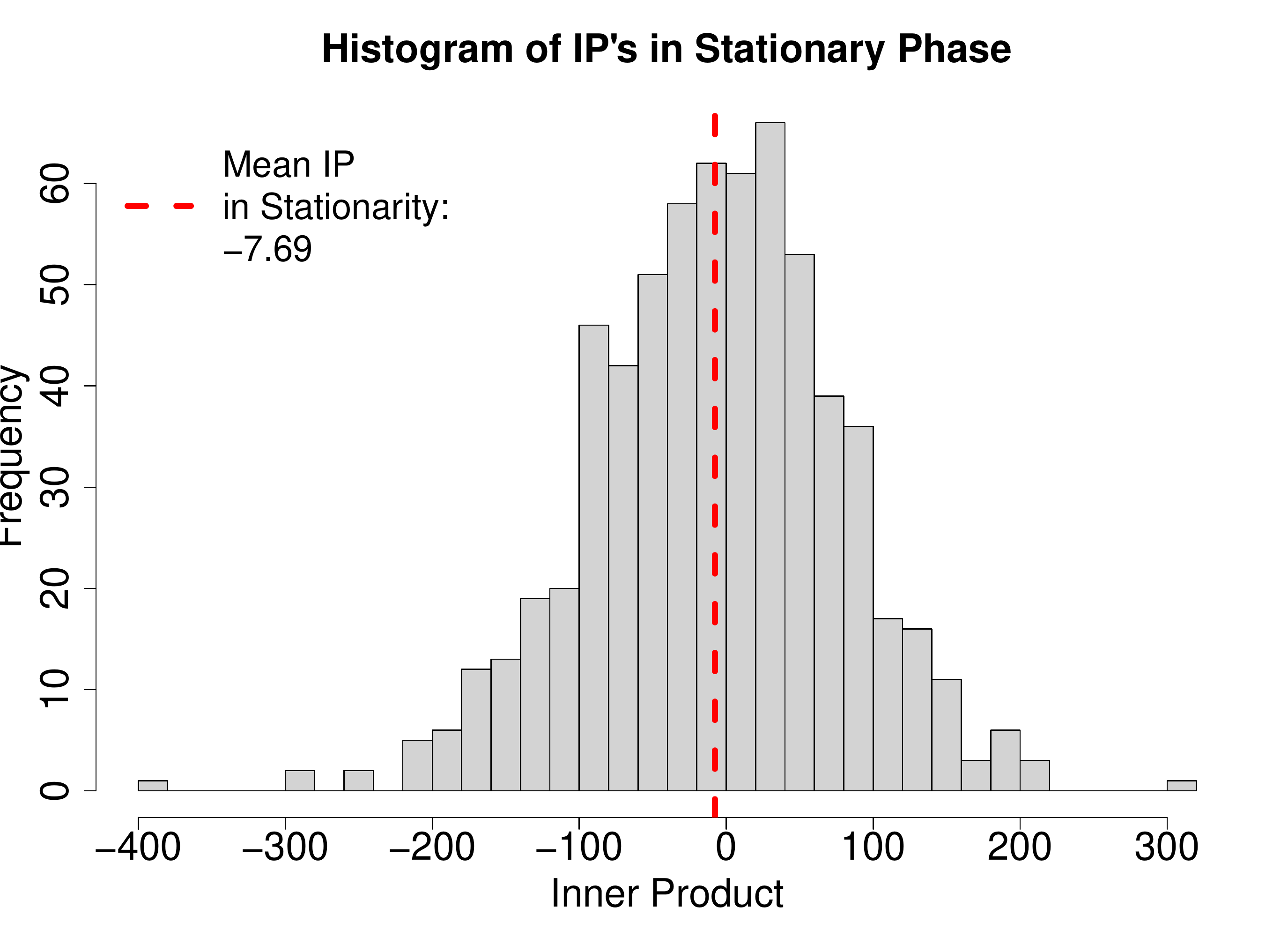}
}
\vspace{-0.2in}
  \caption{Left panel:  Histogram of the inner product of successive gradients (Eq.~(\ref{eq:ip_loss})) for SGDM in the transient phase. Right panel:
  Histogram of the inner products for SGDM in the stationary phase.
  Training settings from the low $\beta$ quadratic setting in Section~\ref{sec:quadratic}.
  Note the asymmetry in the distribution of inner products in the stationary phase.
  }
\label{fig:hist_ip}
\end{figure}

\begin{figure}[h]
\mbox{\hspace{-0.1in}
  \includegraphics[width=1.85in]{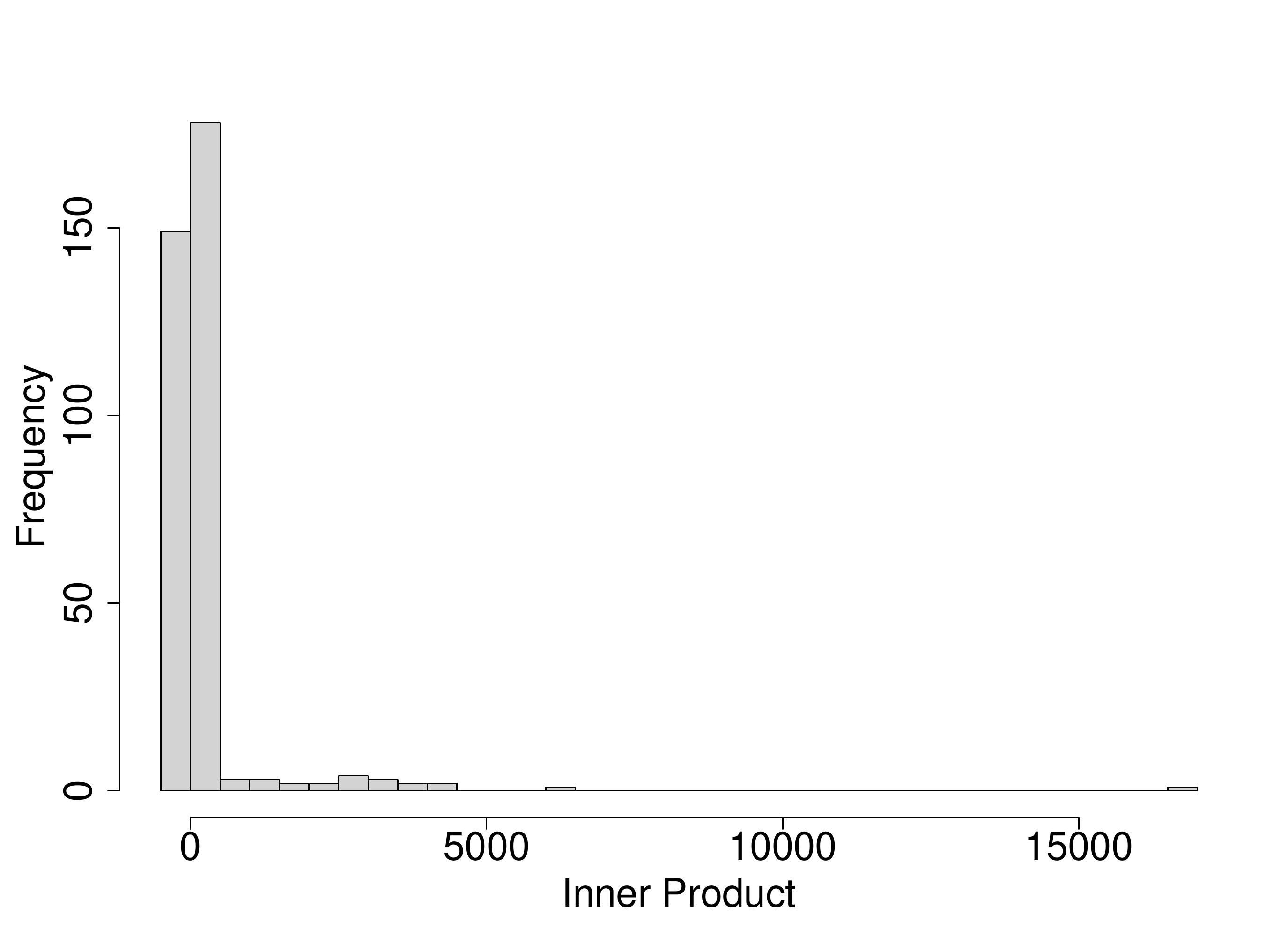}\hspace{-0.12in}
  \includegraphics[width=1.85in]{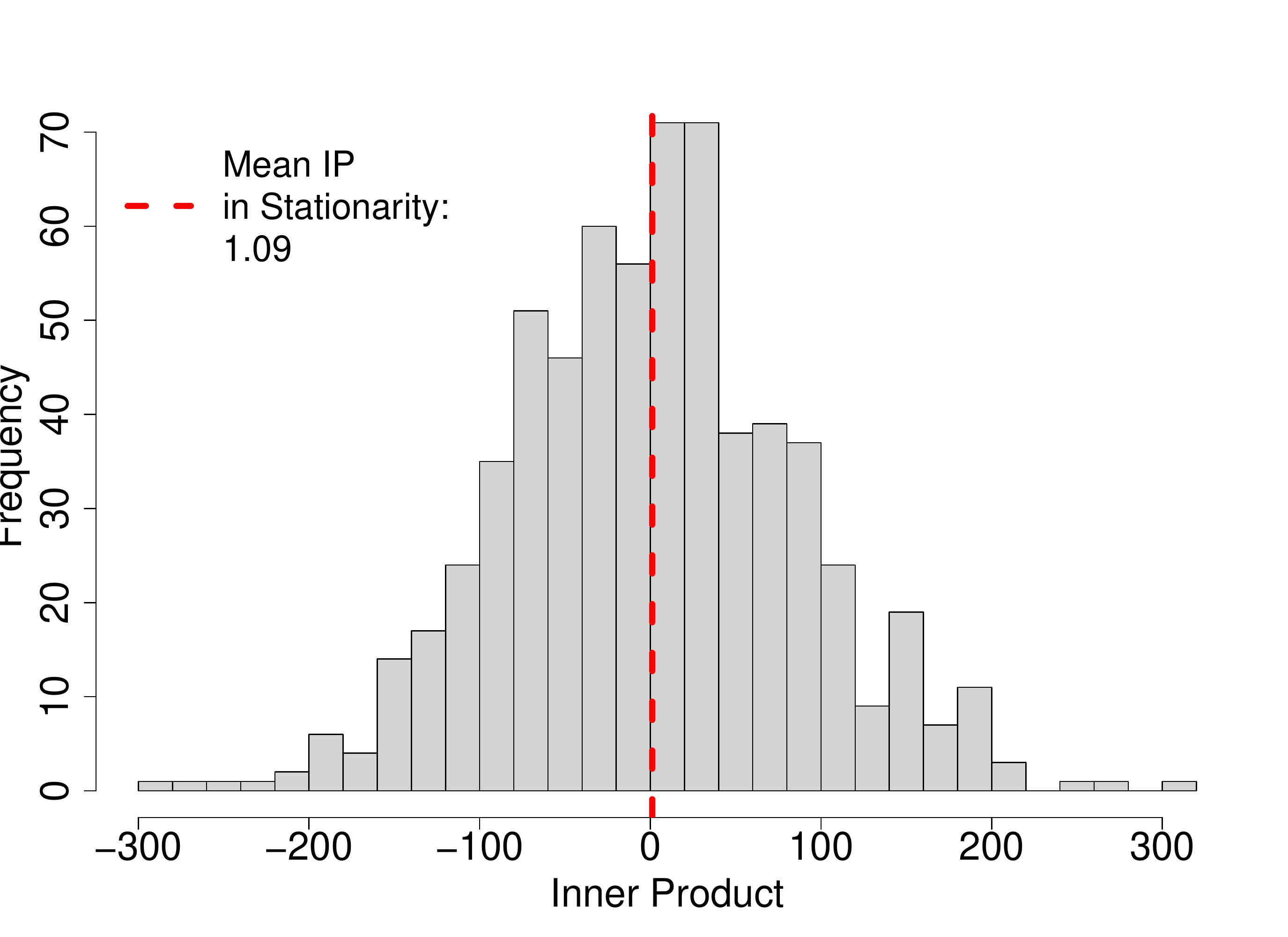}
}
\vspace{-0.2in}
  \caption{ Histogram of the inner product of successive gradients with high momentum $\beta=0.8$ and no momentum reduction.
  Left panel: Transient phase.
  Right panel: Stationary phase.
 Quadratic loss model from Section~\ref{sec:quadratic}.
  Note the symmetry in the distribution of inner products in the stationary phase.
  }
\label{fig:hist_ip_highmom}
\end{figure}

We first provide empirical evidence to support Proposition~\ref{prop:noise_key_iter}.  The low $\beta$ quadratic setting from Section~\ref{sec:quadratic} is used, and SGDM run for 20 epochs. In Fig.~\ref{fig:hist_ip} are histograms of the inner products from Eq.~(\ref{eq:ip_loss}) in the transient and stationary phase. The phase transition is chosen by monitoring MSE with respect to $\theta_\star$.
These results have been observed to be robust across loss functions and parameter settings.  In the transient phase the inner products are majority positive and have positive skew.
This makes sense as when $\theta_n$ is far away from $\theta_\star$, the bias dominates and gradients are likely pointed in the same direction. In the stationary phase we observe a unimodal distribution around zero with a longer negative tail.
\textbf{What is interesting  is the distribution of inner products in the stationary phase.}
The expectation is  negative, and consistently so across many experiments. But the magnitude of the variance exceeds the magnitude of the mean by an order of magnitude. The  high magnitude variance indicates a high frequency of iterates with mean zero inner product.
The larger negative tail--specifically the small number of inner products around $-400$--supports the existence of a small number of key iterates as stated in Proposition~\ref{prop:noise_key_iter}.
This empirical observation is even more striking when you compare similar histograms when the momentum is high $\beta = 0.8$ and keeping all other settings the same, in Fig.~\ref{fig:hist_ip_highmom}.
The previously observed asymmetry in the stationary phase is now gone.
Thus with high momentum (or without momentum reduction), the convergence diagnostic cannot work as there is no clear signal from the test statistic.

\begin{figure}[h]
\mbox{
  \includegraphics[width=1.8in]{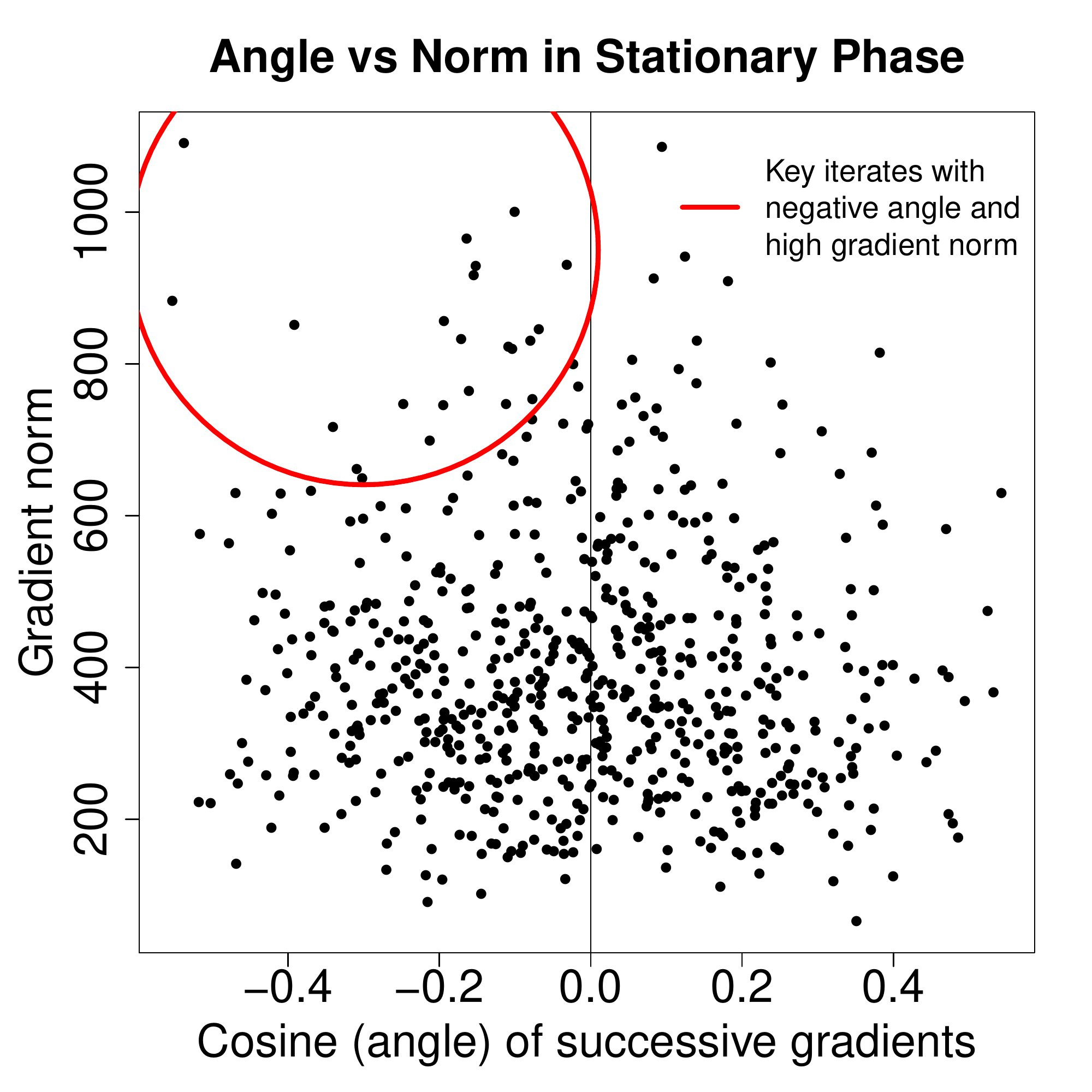}\hspace{-0.12in}
  \includegraphics[width=1.8in]{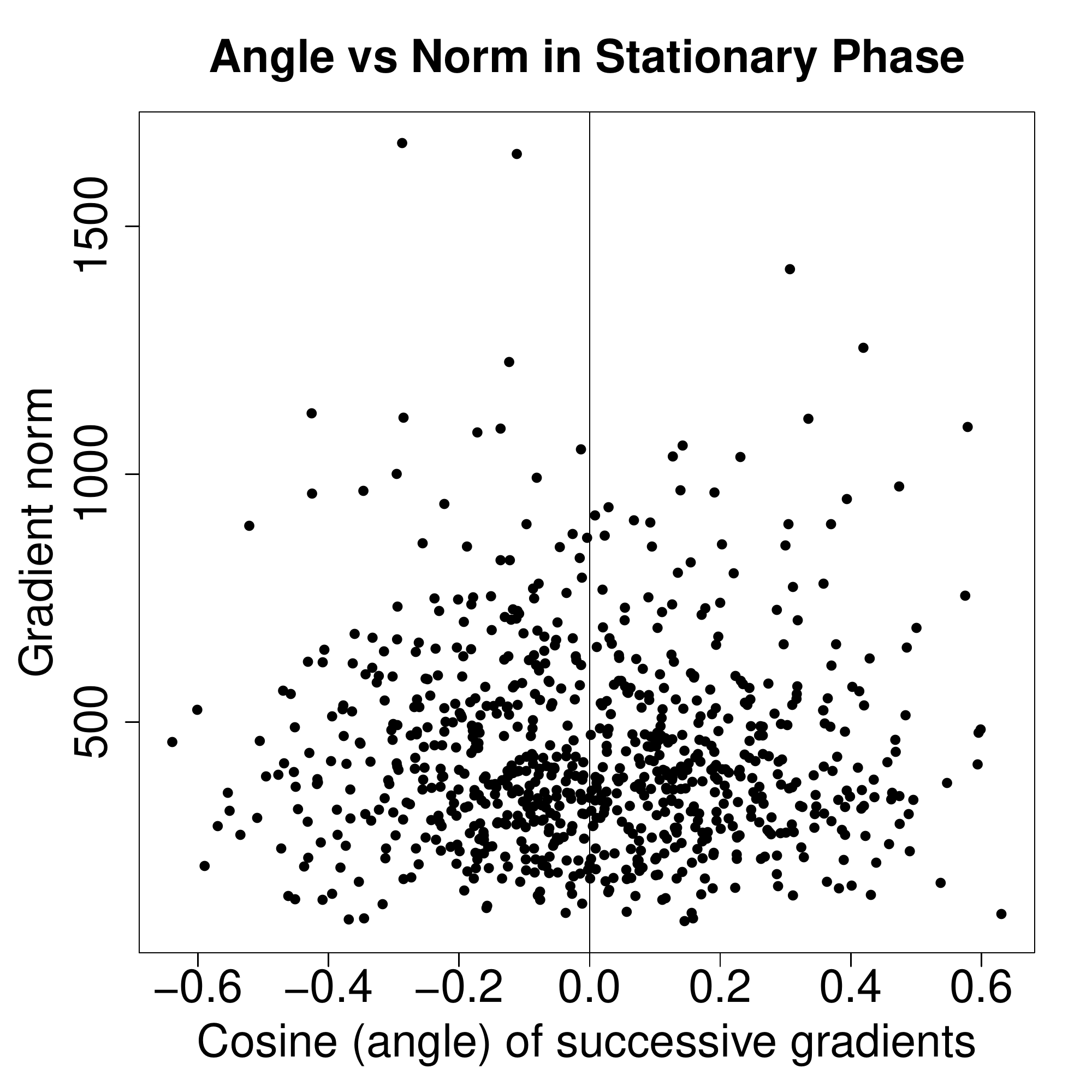}
}
\vspace{-0.2in}
  \caption{Cosine similarity vs gradient norm for SGDM in the stationary phase.
  The red circle indicates those key inner products with negative angle and high gradient norm.
  Left panel: Training settings from the low $\beta$ quadratic setting in Section~\ref{sec:quadratic}.
  Right panel: high momentum $\beta = 0.8$.
  The inner product distribution is symmetric in this case.
  }
  \label{fig:angle_norm}
\end{figure}

Fig.~\ref{fig:angle_norm} provides the further empirical evidence  by plotting the magnitude $\| \nabla \ell (\theta_n, \xi_{n+1}) \|_2^2$ and  cosine$(\nabla \ell (\theta_n, \xi_{n+1}), \nabla \ell (\theta_{n-1}, \xi_{n}) )$ of successive gradients for SGDM in the stationary phase. Again, we first look at the low $\beta$ quadratic setting in Section~\ref{sec:quadratic}. A red circle is drawn to identify iterates with high magnitude and negative angle, exactly those key iterates described in Proposition~\ref{prop:noise_key_iter}.
We see such key iterates exist and drive the expectation negative.
These results have been observed to be robust across loss functions and parameter settings.  With high momentum we have empirically observed that these key inner products with high magnitude and negative angle disappear.

Thus Proposition~\ref{prop:noise_key_iter} helps explain the observed sensitivity of the expected test statistic to high momentum.

\subsection{Variance bounds}
We now provide theory which supports Proposition~\ref{prop:noise_key_iter}. We show that in the stationary phase, the magnitude of the variance dominates the magnitude of the mean
for the test statistic of the convergence diagnostic.  A relatively large variance of the inner products $\nabla \ell (\theta_n, \xi_{n+1})^\top \nabla \ell (\theta_{n-1}, \xi_n)$ suggests that a majority of iterates are dominated by the variance of the stochastic gradient.
A relatively small mean for the inner products suggests that a minority of iterates drive the expectation. Even though in the stationary phase SGDM is trapped in a bounded region, and the expected test statistic driven by a sparse number of key iterates, there is still significant room for random motion.

\begin{theorem}
\label{thm:ip_var_bound}
Consider the SGDM procedure in Eq.~(\ref{eq:sgdm}).
Suppose that Assumptions~\ref{assump:strcvx},~\ref{assump:Lsmooth},~\ref{assump:Fbound},~\ref{assump:min_noise}, and~\ref{assump:scaling} hold. Define\\ $IP = \nabla \ell(\theta_n, \xi_{n+1})^\top \nabla \ell(\theta_{n-1}, \xi_n)$.
Then,
\begin{align*}
\frac{Var [ IP ]}{\mathbb{E} [ IP ]^2}
&\geq \frac{ ( M - L \gamma \sigma_0^2 A_\beta )^2 }{M^2 ( 1 + 8 L / c )^2 }
- 1.\hspace{1in}\Box
\end{align*}\\
\end{theorem}

\begin{corollary}
\label{cor:set_gamma_threshold}
Consider the SGDM procedure in Eq.~(\ref{eq:sgdm}).
Fix a scaling factor $\lambda > 2$.
Set the learning rate $\gamma = 2 t M / L \sigma_0^2 A_\beta$ with $t  \geq 1 + \sqrt{ \lambda } ( 1 + 4 L / c )$.
Then,
\begin{equation*}
\hspace{0.5in}Var [IP] \geq (\lambda - 1) \ \Ex{IP}^2.\hspace{1.2in}\Box
\end{equation*}
\end{corollary}

\emph{Remarks.}
The results hold regardless of the sign of the expectation of inner products, and show that the variance of the test statistic upper bounds the squared mean.
Corollary~\ref{cor:set_gamma_threshold} specifies that a greater learning rate increases the variance bound.
This makes sense as a larger learning rate increases the radius of the stationary region.

We have seen that Theorem~\ref{thm:ip_var_bound} and Corollary~\ref{cor:set_gamma_threshold}, along with Figures~\ref{fig:hist_ip}, \ref{fig:hist_ip_highmom} and~\ref{fig:angle_norm}, provide theoretical and empirical support for Proposition~\ref{prop:noise_key_iter}. While the bound in Theorem~\ref{thm:ip_var_bound} is more robust to $\beta$, it is still unable to provide practical guidance as  the data dependent constants $M$, $L$, $c$, $\sigma_0^2$, and $A_\beta$ must still be estimated.

The convergence diagnostic monitors a certain signal in the gradients.
In Section~\ref{sec:momentum-stationarity} we have shown that this signal can be sensitive to high momentum, and in this Section we have shown that the signal may be sparse within other gradient noise.
Currently, we believe that an empirical mean is still the best way to capture this gradient signal, with the simple but effective automatic reduction in Alg.~\ref{alg:diagnostic} to combat the negative effects of high momentum.

\section{Numerical experiments}\label{sec:synth-data-experiments}
We now evaluate the convergence diagnostic in Alg.~\ref{alg:diagnostic} on synthetic data experiments in settings of quadratic loss and phase retrieval~\cite{Article:Chen_MP19}.
The quadratic loss setting is described in Section~\ref{sec:quadratic}. For phase retrieval let $\ell ( \theta, y, x ) = 1/4 [ ( x^\top \theta )^2 - y ]^2$ with $x \sim N ( 0, I_{20} )$ and $y = ( x^\top \theta_\star )^2$.
$\theta_{\star,i} = (-1)^i \times 2 \exp( -0.7 i )$ for $i = 1, \dots, 20$.
The number of data points $N = 10^3$.
The checking period $c$ is every epoch.
Due to non-convexity in phase retrieval we only record the training runs where SGDM has entered a good minima to eliminate the need to tune  the parameters for our error rate procedure for different minima.

There are two failure modes: the convergence diagnostic can activate too early, or too late.
Let $\theta_n$ be the estimate when the convergence diagnostic has activated.
If the diagnostic activates too early then the error is too high, i.e., $\| \theta_n - \theta_\star \|^2 > \eta$ for some threshold value.
$\eta$ is set as a tight upper bound on the error observed in the stationary phase across many runs.
If the diagnostic activates too late, it can waste unnecessary computation. Let $K = (n - k) / n$ such that $\| \theta_k - \theta_n \|^2 = \eta$ and $k \leq n$.
If the diagnostic activates too late, then we expect $\theta_n$ to be far into the stationary phase,  and thus $n-k$ to be a significant portion of  $n$, i.e., $K > \kappa$ for some threshold value. $\kappa$ is set as a tight lower bound on the $K$ calculated by running SGDM  into the stationary phase.

Table~\ref{tab:diag_errors} displays the results of 100 independent runs with the percentage of type I errors (too early), type II errors (too late), and good diagnostic activations.
Low and high momentum settings are used for quadratic loss and phase retrieval.
Empirically the type I errors are small, while the type II errors are a larger concern.
This observation on the higher frequency of type II errors is corroborated by~\cite{Proc:Chee_AISTATS18}. Encouragingly we see that the automatic momentum reduction in Alg.~\ref{alg:diagnostic} has enabled the convergence diagnostic to be robust to momentum.
High momentum settings have little to no effect on the Type I and II error rates of the convergence diagnostic.
Contrast Table~\ref{tab:diag_errors} with the results in Table~\ref{tab:high_momentum_ip}, where the sign of the test statistic in the stationary phase was shown to be sensitive to momentum.
Additionally, in approximately $50$-$70\%$ of all Type II errors the diagnostic activated only moderately late. In approximately $10$-$20\%$ of Type II errors the diagnostic~activated~late.

\begin{table}[t]
\caption{
Empirical evaluation of the convergence diagnostic in Alg.~\ref{alg:diagnostic} over 100 independent runs for each experimental setting.
SGDM run for $20$ epochs with batch size $20$.
Quadratic low $\beta$ (Q-Low) set $\beta = 0.2$, $\gamma = 10^{-2}$, $\eta = 10^{-3}$, $\kappa = 0.65$.
Quadratic high $\beta$ (Q-High) set $\beta = 0.8$, $\gamma = 10^{-2}$, $\eta = 2 \times 10^{-3}$, $\kappa = 0.30$.
Phase retrieval low $\beta$ (PR-Low) set $\beta=0.2$, $\gamma=10^{-2}$, $\eta = 10^{-2}$, $\kappa = 0.6$.
Phase retrieval high $\beta$ (PR-High) set $\beta=0.8$, $\gamma=10^{-2}$, $\eta = 10^{-2}$, $\kappa = 0.65$.
\vspace{-0.1in}}
\label{tab:diag_errors}
\begin{center}
\begin{tabular}{|c| c c c|}
\hline
 & Type I & Type II & Good \\ 
 & (too early) & (too late) & activation \\
\hline
Q-Low & 1\% & 22\% & 77\% \\
\hline
Q-High & 0\% & 17\% & 83\% \\
\hline
PR-Low & 1\% & 17\% & 82\% \\
\hline
PR-High & 0\% & 16\% & 84\% \\
\hline
\end{tabular}
\end{center}
\end{table}

\section{Application: an automatic learning rate schedule}
\label{sec:autoLR}
The convergence diagnostic has a natural application in automating the learning rate schedule.
Hyper-parameter tuning, especially for the learning rate, has a major effect on the performance of stochastic gradient methods.
Tuning is typically a very manual process requiring many training re-runs.
The benefit of an automatic learning rate is to greatly reduce the amount of supervision and number of training runs.

\begin{algorithm}[h]
\SetKwInOut{Input}{input}
\Input{Alg.~\ref{alg:diagnostic} {\tt SGDM}$(\theta, \gamma)$,  initial and minimum stepsize $\gamma_0$, $\gamma_{min}$, learning rate reduction $\rho \in (0,1)$
}
$\gamma \gets \gamma_0$ \\
\While{$\gamma > \gamma_{min}$}{
	$\theta \gets {\tt SGDM} ( \theta, \gamma )$ \\
	$\gamma \gets \rho \times \gamma$ \\
}
\Return{ $\theta$}
\caption{SGDM with automatic learning rate}
\label{alg:autoLR}
\end{algorithm}

In Alg.~\ref{alg:autoLR} we present an automatic learning rate based on the convergence diagnostic in Alg.~\ref{alg:diagnostic}.
SGDM with constant learning rate moves quickly towards $\theta_\star$ but cannot improve beyond distance $O(\gamma)$, as suggested by Theorem~\ref{thm:mom_convg_analysis}.
The convergence diagnostic is used to detect stationarity, after which the learning rate is reduced $\gamma \gets \rho \gamma$ and a smaller radius $O(\rho \gamma)$ of stationarity achieved, with $\rho \in (0,1)$. Alg.~\ref{alg:autoLR} takes advantage of the speedup afforded to constant rate in the transient phase, while avoiding the trade-off cost of a larger stationary region by reducing the learning rate. In practice it is common to use a constant learning rate and decrease several times at hand chosen points~\cite{Proc:He_CVPR16,Article:Krizhevsky_CACM17}.

A major benefit of automatic hyper-parameter tuning is robustness to a potentially misspecified initial setting, in this case $\gamma_0$. We train logistic regression on benchmark datasets
 MNIST and Online News Popularity (from UCI repository),
for a variety of initial learning rates $\gamma_0$.
In Fig.~\ref{fig:mnist_news_binary} the accuracy on a held out test set is compared between the automatic rate in Alg.~\ref{alg:autoLR} and a decreasing rate $\gamma = \gamma_0 / n$ for $\gamma_0^{mnist} \in \{1.0, 0.1, 0.01\}$ in MNIST and $\gamma_0^{news} \in \{10, 1.0, 0.1\}$ in Online News.
The findings are consistent across both datasets.
The automatic learning rate was significantly more robust to initial conditions than the decreasing rate. In addition, the automatic learning rate achieved significantly higher test accuracy than the~decreasing~rate.

\begin{figure}[t]
\begin{center}
  \includegraphics[width=2.3in]{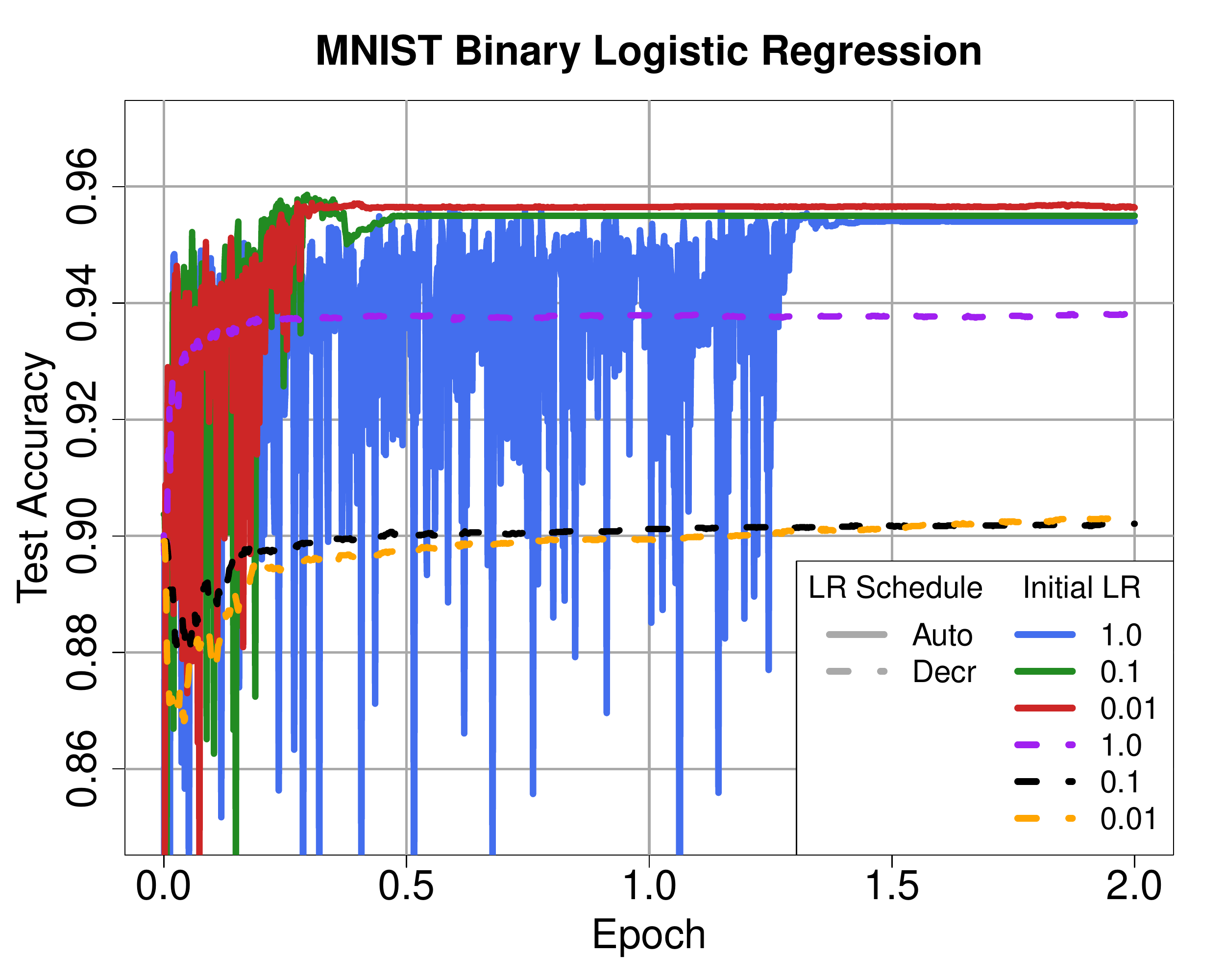}
\includegraphics[width=2.3in]{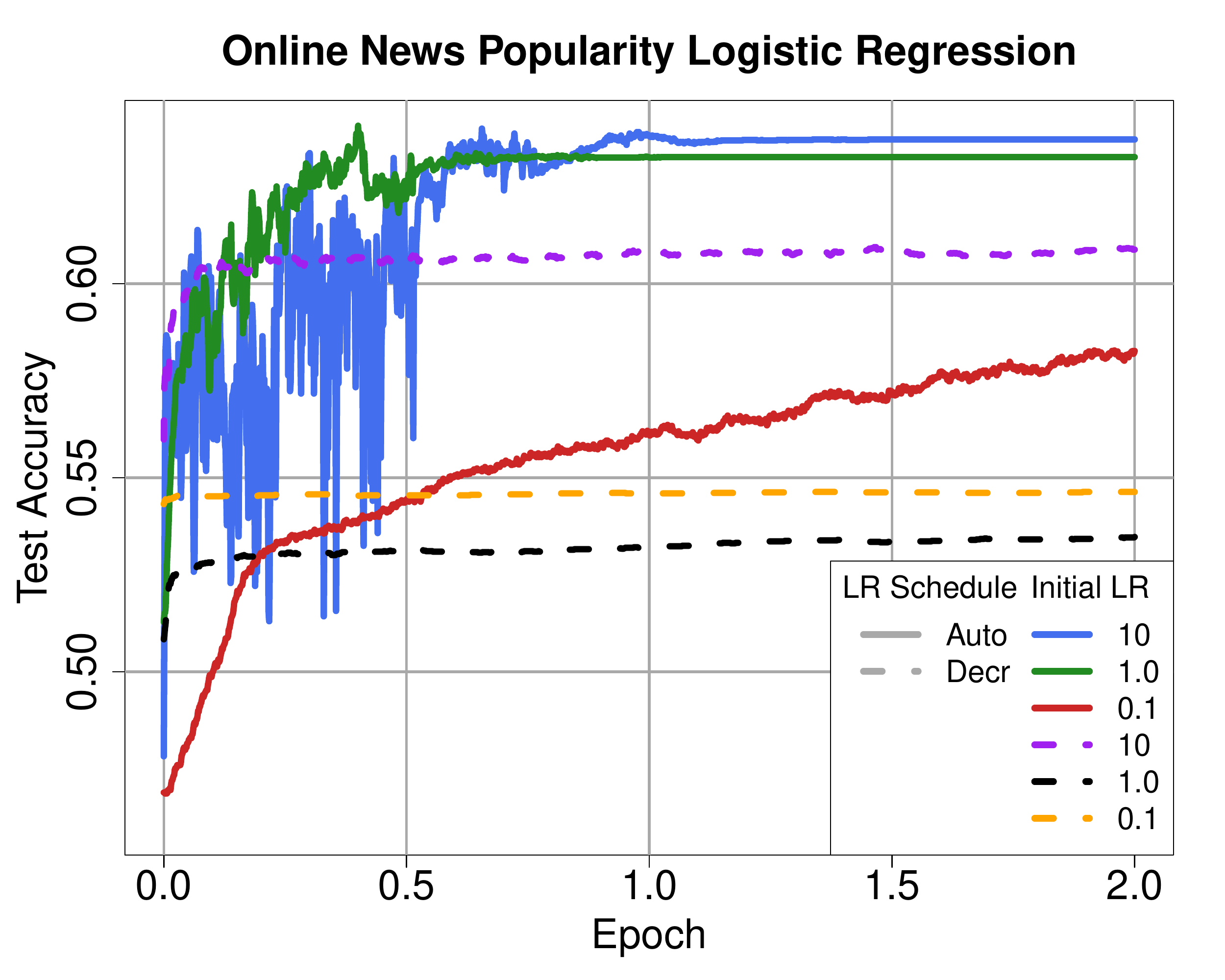}
\end{center}
\vspace{-0.16in}
  \caption{Binary logistic regression  with SGDM using Alg.~\ref{alg:autoLR} and decreasing rate $\gamma = \gamma_0 / n$,
  $\beta = 0.8$. Upper panel: MNIST. Bottom panel: Online News Popularity.
  }
\label{fig:mnist_news_binary}\vspace{-0.13in}
\end{figure}

\begin{figure}[h!]
\begin{center}
  \includegraphics[width=2.3in]{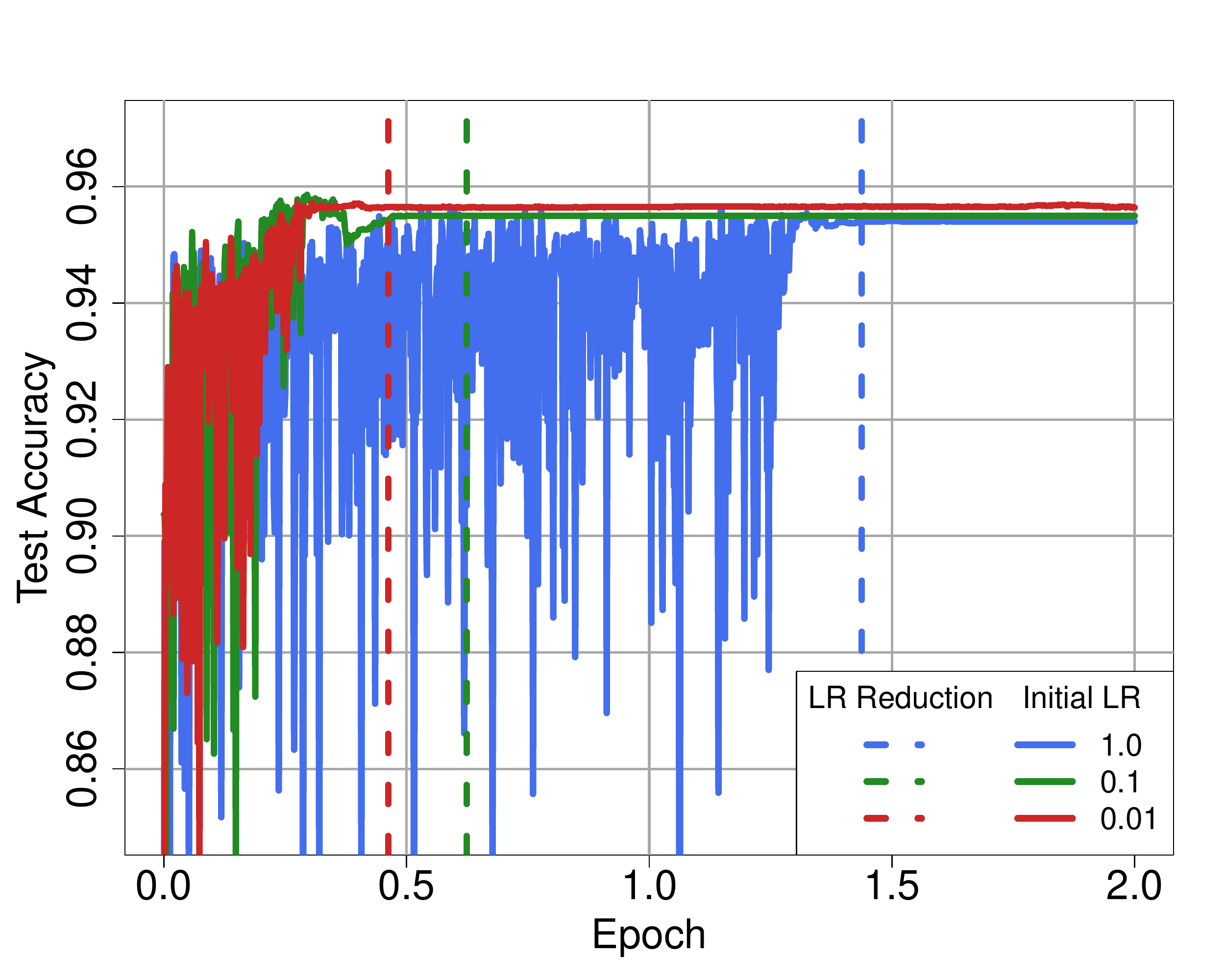}
\includegraphics[width=2.3in]{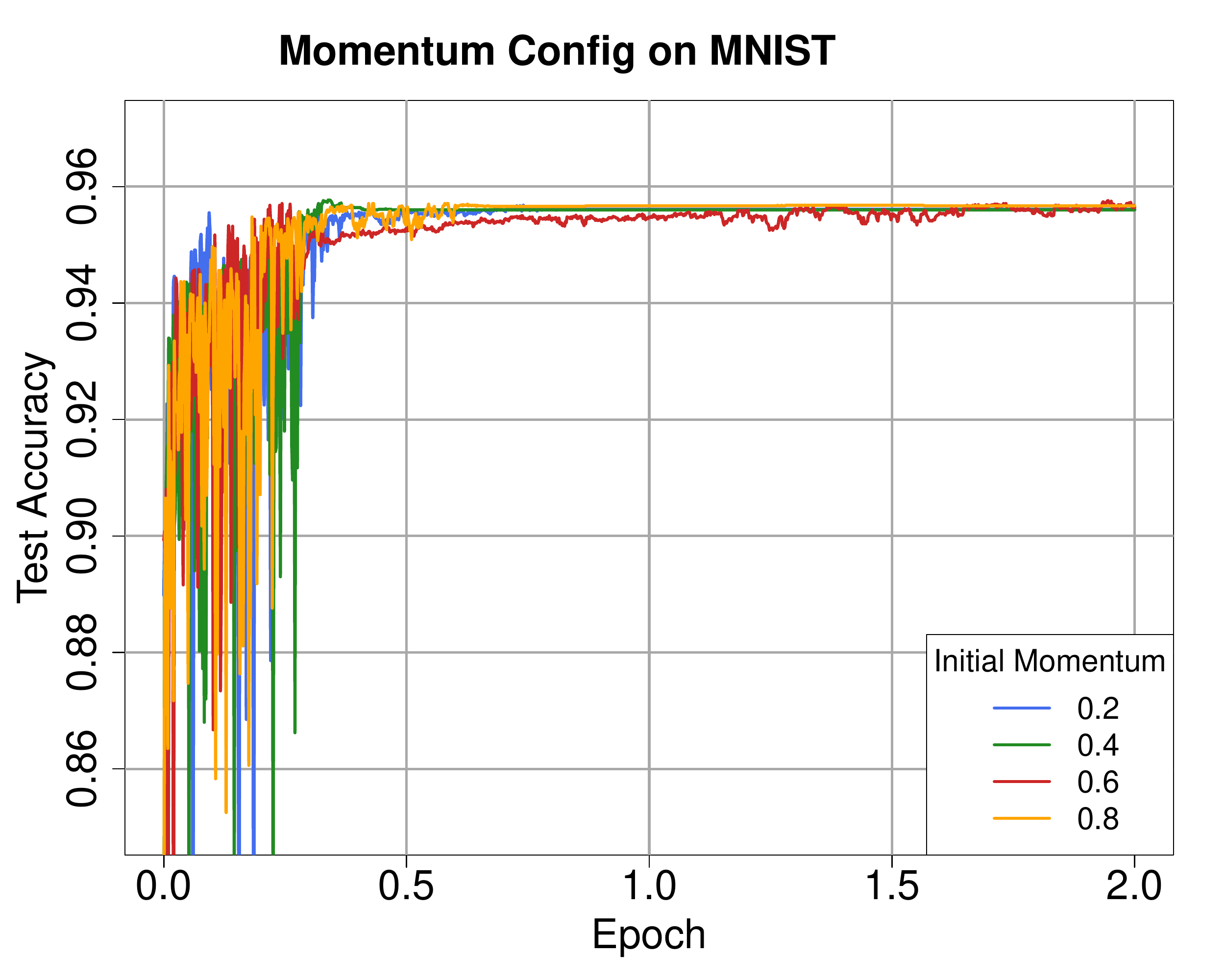}
\end{center}
\vspace{-0.16in}
  \caption{Upper: Vertical lines marks the diagnostic activation and learning rate reduction on MNIST.
  Bottom: SGDM using Alg.~\ref{alg:autoLR} and varying momentum $\beta\in\{0.2, 0.4, 0.6, 0.8\}$.
  }
\label{fig:mnistactivate_momvary}\vspace{-0.12in}
\end{figure}

In Fig.~\ref{fig:mnistactivate_momvary} (upper panel) we plot only SGDM using Alg.~\ref{alg:autoLR} and vertical lines to mark the convergence diagnostic's activation.
On the bottom panel of Fig.~\ref{fig:mnistactivate_momvary} shows SGDM using Alg.~\ref{alg:autoLR} with varying initial momentum.
From the similarity in training curves we see that Alg.~\ref{alg:autoLR} is also robust to setting the momentum.
The function $h()$ we used in Alg.~\ref{alg:diagnostic} to set the change in momentum was the mean squared distance between successive iterates.
However, any convergence heuristic can be used for $h()$ in Alg.~\ref{alg:diagnostic}.
We show that $h()$ works consistently in the upper panel of Fig.~\ref{fig:momred_constlr}.
Regardless of the initial momentum, SGDM using Alg.~\ref{alg:autoLR} on MNIST reduces the initial momentum at a consistent point, helping to ensure a consistent activation of the convergence diagnostic.

\begin{figure}[h]
\begin{center}
  \includegraphics[width=2.3in]{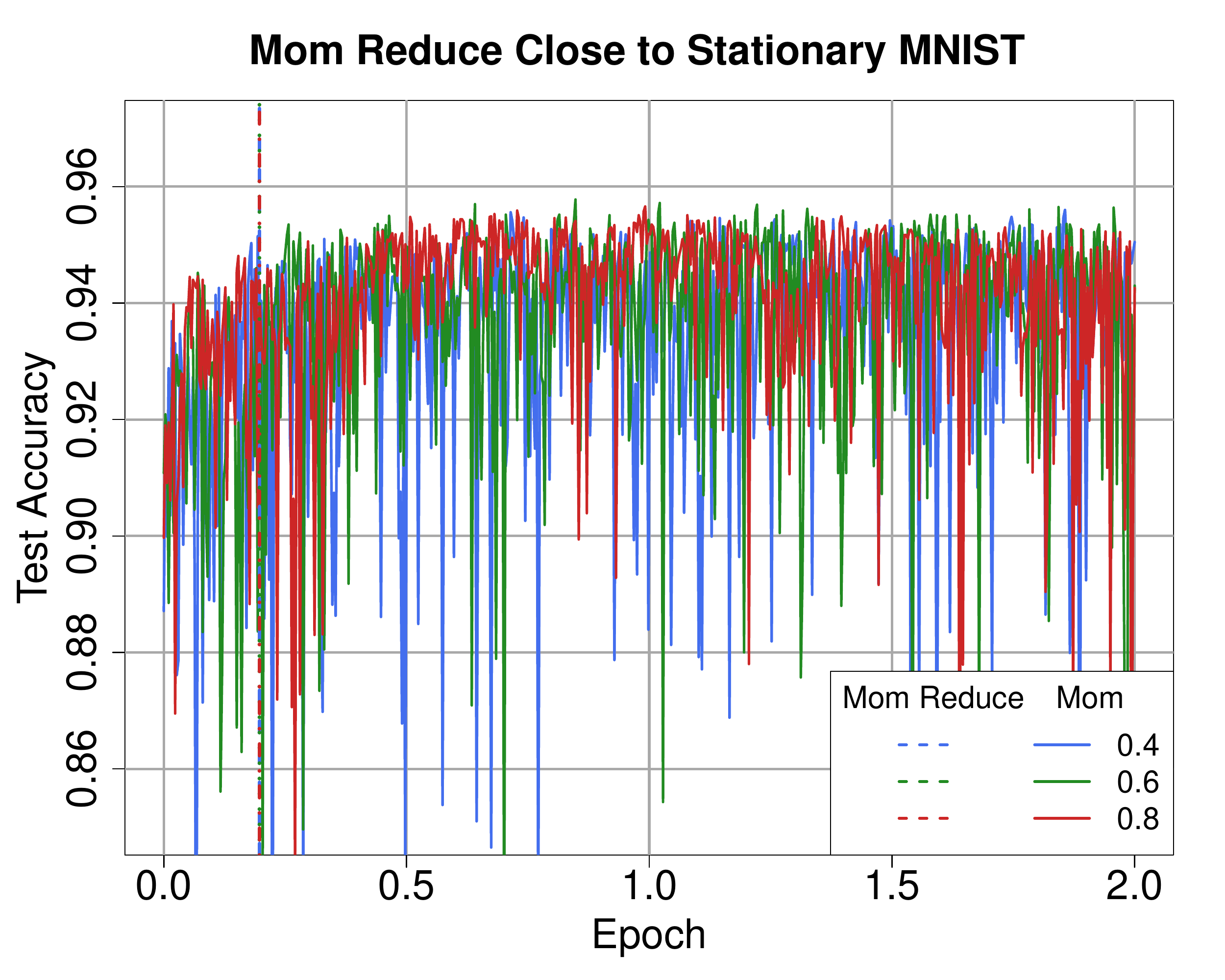}\hspace{-0.12in}
\includegraphics[width=2.3in]{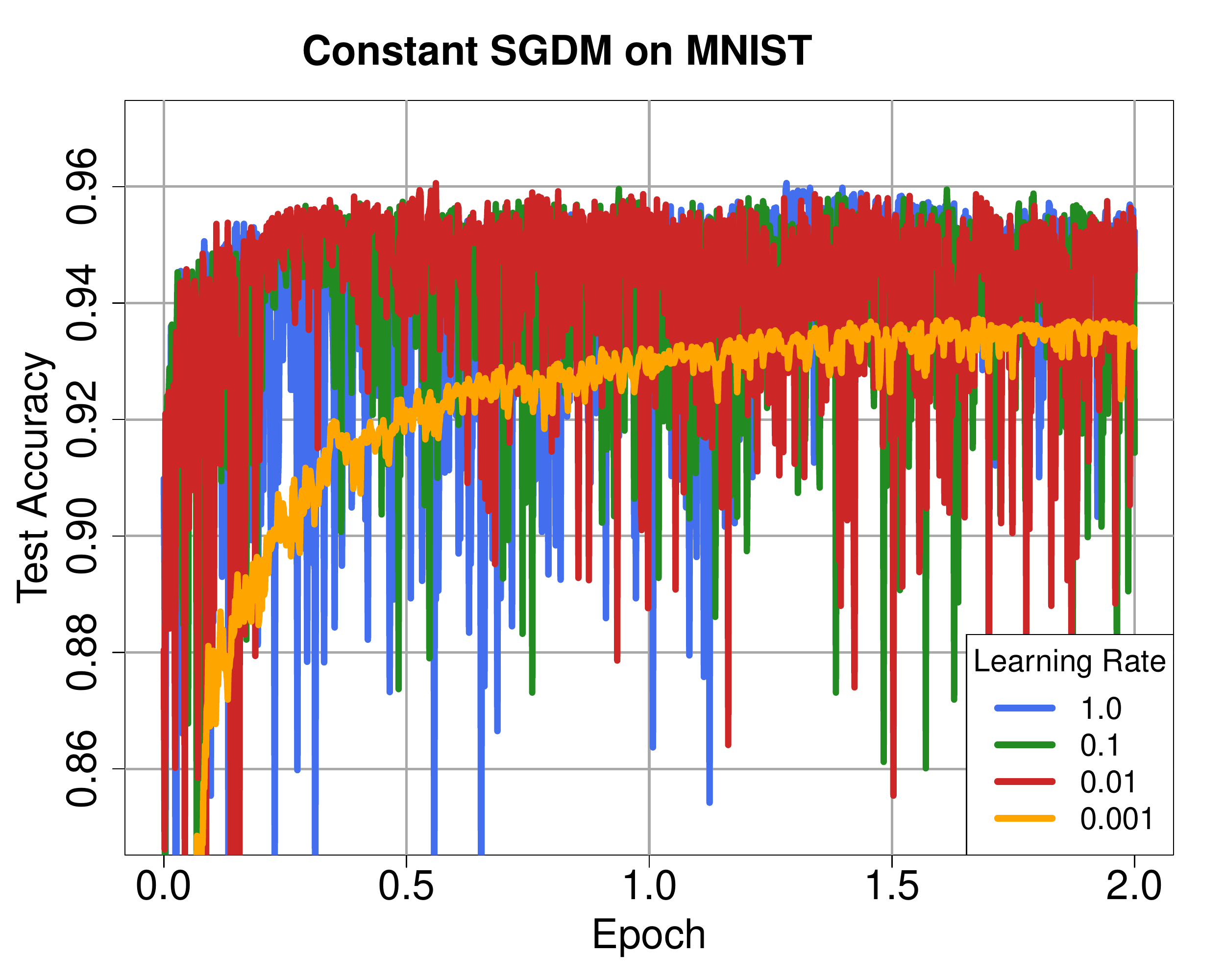}
\end{center}
\vspace{-0.17in}
  \caption{Upper: Vertical line marks a consistent momentum reduction using convergence heuristic of $\| \theta_n - \theta_{n-1} \|^2$.
  Bottom: Binary logistic regression with SGDM using Alg.~\ref{alg:autoLR} and constant learning rate $\gamma_0 \in \{1.0, 0.1, 0.01, 0.001\}$ on MNIST.
  }
\label{fig:momred_constlr}\vspace{-0.15in}
\end{figure}

Experiments were performed with constant rate $\gamma = \gamma_0$, however the stationary region was large enough to result in significant test accuracy fluctuations for a given $\gamma_0$.

\begin{figure}[b]
\vspace{-0.15in}
\begin{center}
\includegraphics[width=2.4in]{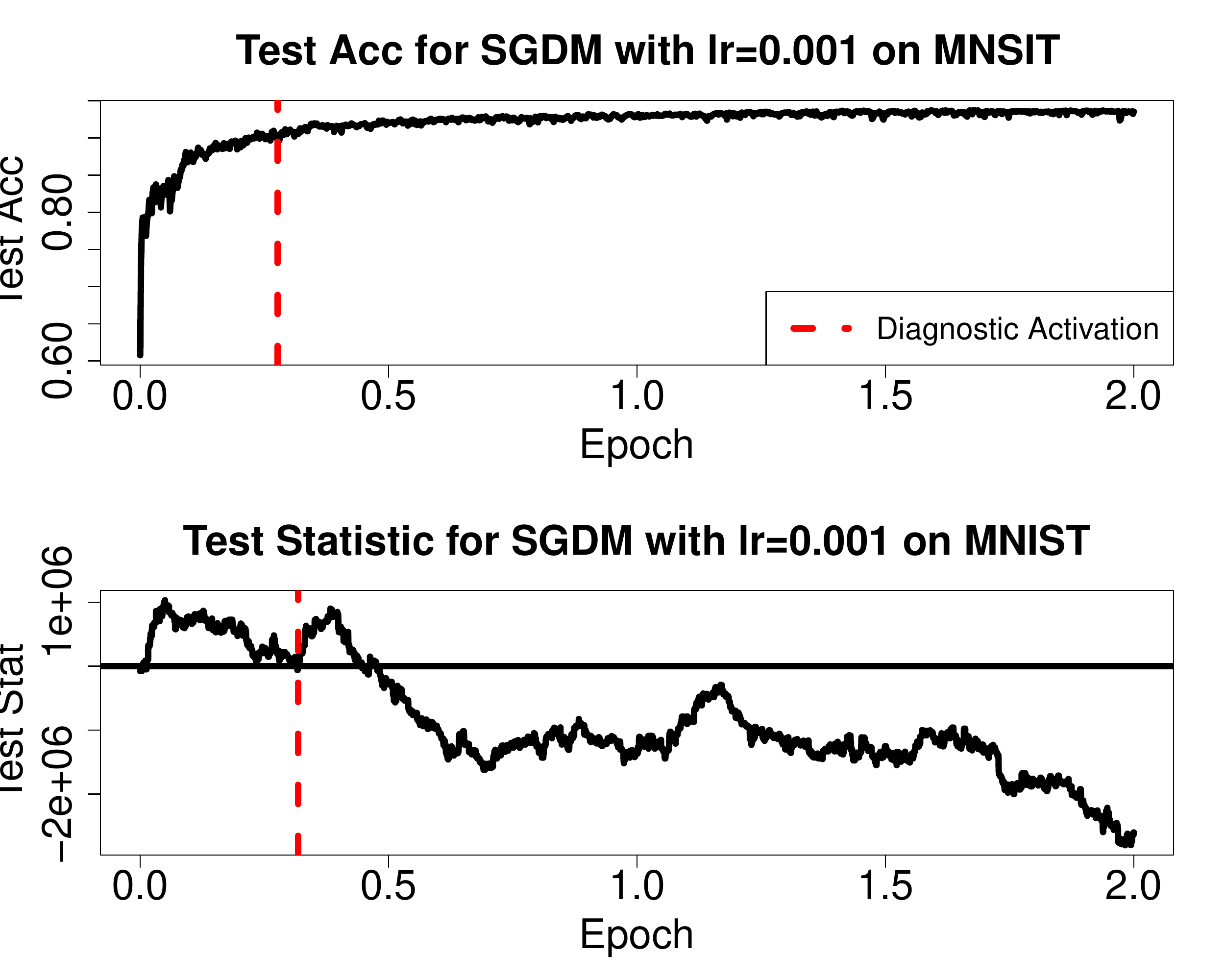}
\end{center}
\vspace{-0.15in}
  \caption{Test accuracy and convergence diagnostic test statistic from Alg.~\ref{alg:diagnostic} on MNIST.
  The vertical line marks the diagnostic's activation, which coincides with the test accuracy flattening out.
  }
\label{fig:mnist_teststat}
\end{figure}

The fact that SGDM using Alg.~\ref{alg:autoLR} is able to achieve competitive performance indicates that the convergence diagnostic is working well.
As further evidence that the convergence diagnostic in Alg.~\ref{alg:diagnostic} is working well, Fig.~\ref{fig:mnist_teststat} plots the test statistic of the convergence diagnostic along with the test accuracy for SGDM on MNIST.
The vertical line indicates the activation of the convergence diagnostic, when the test statistic becomes negative.
The diagnostic activates just as the test accuracy flattens out.

\begin{figure}[h]
\begin{center}
  \includegraphics[width=2.5in]{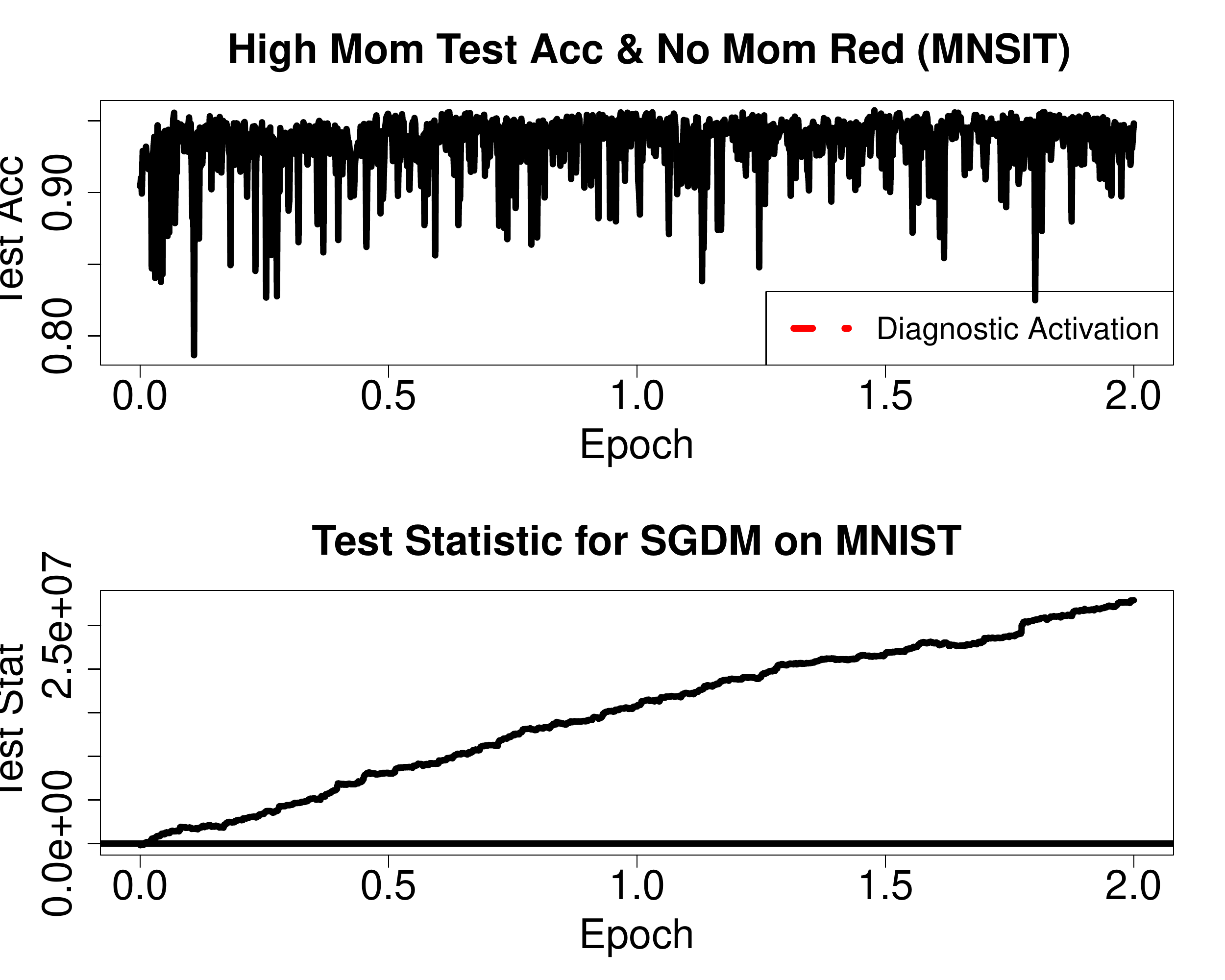}\hspace{-0.12in}
\end{center}
\vspace{-0.15in}
  \caption{Test accuracy and convergence diagnostic test statistic on MNIST. The momentum reduction has been removed, and $\beta = 0.8$. The test accuracy has flattened out and SGDM has converged, but the diagnostic does not activate because the test statistic is perpetually positive.
  }
\label{fig:mnist_ablation}
\end{figure}

We conduct an ablation study to understand the negative effects of high momentum on our convergence diagnostic.
Fig.~\ref{fig:mnist_ablation} plots the test accuracy and test statistic by removing the momentum reduction component of Alg.~\ref{alg:diagnostic} and using SGDM with high momentum $\beta = 0.8$ on MNIST.
SGDM has convergence because the test accuracy has plateaued, but the convergence diagnostic does not activate because its test statistic is perpetually positive.
We plot the test statistic of Alg.~\ref{alg:diagnostic} in Fig.~\ref{fig:mnist_ablation2} with no momentum reduction and increasing momentum $\beta \in \{0.2, 0.4, 0.6, 0.8\}$.
Momentum has a proportional relationship with the slope of the test statistic.
A positive slope indicates that the test statistic is not negative upon convergence, and thus the convergence diagnostic without the momentum reduction is ineffective for higher momentum $\beta \in\{ 0.6, 0.8\}$ on MNIST.

\begin{figure}[h!]
\begin{center}
\includegraphics[width=2.3in]{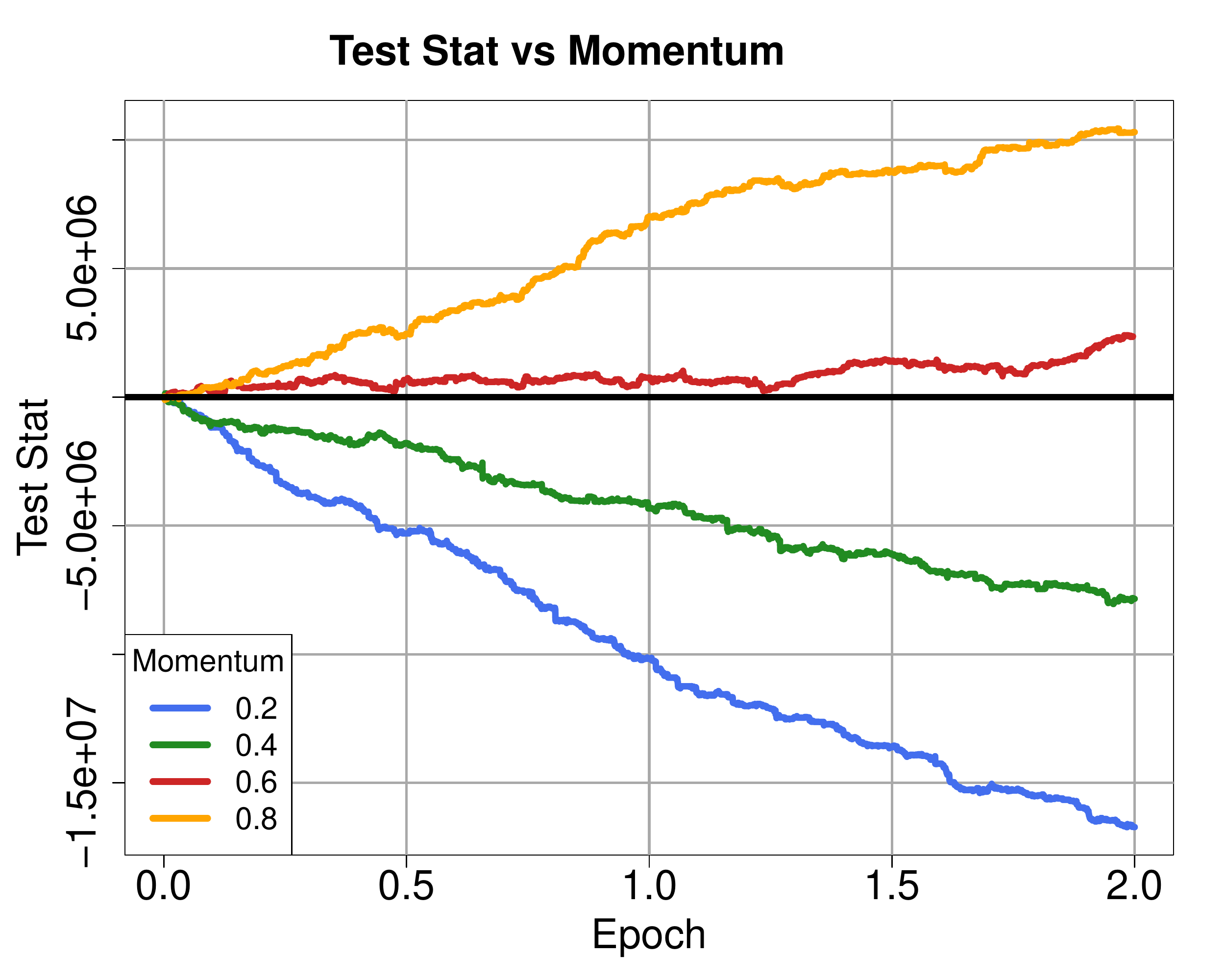}
\end{center}
\vspace{-0.15in}
  \caption{
  Test statistic for different values of momentum $\beta \in \{0.2, 0.4, 0.6, 0.8\}$.
  Higher momentum increases the slope of the test statistic, indicating an even greater difficulty for the convergence diagnostic to detect the stationary phase.
  }
\label{fig:mnist_ablation2}
\end{figure}

\section{Conclusion}

In this paper we focus on detecting the phase transition of SGDM (stochastic gradient descent with momentum) to the stationary phase.
Inspiration is drawn from literature on stopping times in stochastic approximation.
Momentum introduces challenges in the construction and operation of the test statistic for the convergence diagnostic.
We present theory and experiments which support that high momentum alters the  trajectory of the stochastic gradients which the diagnostic monitors. In addition we show the dynamics of SGDM in stationarity are largely random with a sparse number of key iterates behaving in an informative way, which the diagnostic is able to capture.
The proposed automatic momentum reduction technique resolves the issues with high momentum.
Empirical results demonstrate that the diagnostic has few type I errors and a reasonably small number of type II errors, and thus reliably detects convergence to the stationary phase.
We present an application to an automatic learning rate and show that it is robust to initial conditions.

Future work include extensions to other stochastic gradient methods such as adaptive gradient  or distributed methods such as ``local method''~\cite{Proc:Chen_FODS20} are  of great interest.\\

\vspace{0.35in}

\bibliographystyle{plain}
\bibliography{standard}

\newpage\clearpage
\onecolumn
\section{Proofs for Section~\ref{sec:momentum-stationarity}}

\noindent We re-state the assumptions:

\vspace{0.05in}

\noindent\textbf{Assumption 1.} \ The expected loss $f(\theta) = \Ex{ \ell(\theta, \xi) }$ is strongly convex with constant $c$.

\vspace{0.05in}

\noindent\textbf{Assumption 2.} \ The expected loss $f(\theta) = \Ex{ \ell(\theta, \xi) }$ is Lipschitz-smooth with constant $L$.

\vspace{0.05in}

\noindent\textbf{Assumption 3.} \ Theorem~\ref{thm:mom_convg_analysis}~\cite{Report:Yang_arXiv16} holds s.t. $\Ex{ f(\theta_n) - f(\theta_\star) } \leq \gamma M$ for some $M > 0$ and large enough $n$.

\vspace{0.05in}

\noindent\textbf{Assumption 4.} \ $\exists \ \sigma_0^2 > 0$ s.t. $\Ex{ \| \nabla \ell (\theta, \xi) \|^2 } > \sigma_0^2$.

\vspace{0.05in}

\noindent\textbf{Assumption 5.} \ $\exists \ K > 1$ s.t. $\Ex{ (\theta_n - \theta_{n-1})^\top (\theta_{n-1} - \theta_{n-2}) } \geq - K \Ex{ \| \theta_n - \theta_{n-1} \|^2 }$ for large enough $n$.\\

Firstly, we  derive a lower bound on the expected distance between iterates.
\begin{lemma}
\label{lemma:lower_bound_thetadiff}
Suppose that Assumptions~\ref{assump:min_noise} and~\ref{assump:scaling} hold.
Then for large enough $n$,
\begin{equation*}
\mathbb{E} [ \| \theta_n - \theta_{n-1} \|^2 ] \geq
\gamma^2 \sigma_0^2 \biggl( \frac{1}{1 + 2 K \beta + \beta^2} \biggr)
\end{equation*}
\end{lemma}

\begin{proof}
By re-arranging the update equation for SGDM in Eq.~(\ref{eq:sgdm}),
we get $\theta_n - \theta_{n-1} - \beta ( \theta_{n-1} - \theta_{n-2} ) = - \gamma \nabla \ell ( \theta_{n-1}, \xi_n )$.
For brevity, let $\Delta_n \equiv \theta_n - \theta_{n-1}$, and $\nabla \ell_n \equiv \nabla \ell ( \theta_{n-1}, \xi_n )$.
Applying squared norm and re-arranging terms,
\begin{align*}
\| \Delta_n - \beta \Delta_{n-1} \|^2 &= \| - \gamma \nabla \ell_n \|^2 \\
\| \Delta_n \|^2 &= \gamma^2 \| \nabla \ell_n \|^2 + 2 \beta \Delta_n^\top \Delta_{n-1} - \beta^2 \| \Delta_{n-1} \|^2
\end{align*}

Apply expectations to both sides and the second inequality assumption to define a recursive relation.
\begin{align*}
\Ex{ \| \Delta_n \|^2 } &= \gamma^2 \Ex{ \| \nabla \ell_n \|^2 } + 2 \beta \Ex{ \Delta_n^\top \Delta_{n-1} } - \beta^2 \Ex{ \| \Delta_{n-1} \|^2 } \\
&\geq \gamma^2 \sigma_0^2 - \Ex{ \| \Delta_{n-1} \|^2 } ( 2 \beta K + \beta^2 )
\end{align*}

Using this recursive relation we obtain
\begin{align*}
\Ex{ \| \Delta_n \|^2 } &\geq \gamma^2 \sum_{i=0}^{n-1} ( -1 )^{i} ( 2 \beta K + \beta^2 )^{i} \sigma_0^2
= \gamma^2 \sigma_0^2 \biggl( \frac{1 - ( 2 \beta K + \beta^2 )^n}{1 + 2 \beta K + \beta^2} \biggr)
\end{align*}

Because we expect the number of iterations $n$ to be large when entering the stationary phase, we make the approximation
\begin{align*}
\Ex{ \| \Delta_n \|^2 } &\geq \gamma^2 \sigma_0^2 \biggl( \frac{1}{1 + 2 \beta K + \beta^2} \biggr)
\end{align*}
\end{proof}
\emph{Remarks.}
If $\Ex{ \Delta_n^\top \Delta_{n-1} } > 0$ then any negative lower bound is trivial.
It is reasonable to assume $K$ is not too large.
We expect $\| \theta_{n} - \theta_{n-1} \|^2 \approx \| \theta_{n-1} - \theta_{n-2} \|^2$ because they are successive iterates. \\

\begin{customthm}{\ref{thm:ip_exp_bd}}
Suppose that Assumptions~\ref{assump:strcvx},~\ref{assump:Fbound}, and~\ref{assump:min_noise} hold.  The test statistic for the convergence diagnostic in Algorithm~1 for SGDM in Eq.~(\ref{eq:sgdm})
is bounded
\begin{equation*}
\mathbb{E} [ \nabla \ell ( \theta_n, \xi_{n+1} )^\top \nabla \ell ( \theta_{n-1}, \xi_n ) ]
\leq ( 1 + \beta ) \left[ M - \frac{c}{2} \gamma \sigma_0^2 A_\beta \right]
< 0
\end{equation*}
as $n \rightarrow \infty$.
And thus the convergence diagnostic activates almost surely.
\end{customthm}

\begin{proof}
First re-write the inner product with the decomposition in Eq.~(\ref{eq:ip_breakdown}).
\begin{align*}
\nabla \ell ( \theta_n, \xi_{n+1} )^\top \nabla \ell ( \theta_{n-1}, \xi_n )
&= \frac{1}{\gamma} \nabla \ell ( \theta_n, \xi_{n+1} )^\top ( \theta_{n-1} - \theta_n )
+ \frac{\beta}{\gamma}  \nabla \ell ( \theta_n, \xi_{n+1} )^\top ( \theta_{n-1} - \theta_{n-2} )
\end{align*}

Apply expectation to both sides, and then apply the strong convexity assumption.
\begin{align*}
\Ex{ \nabla \ell ( \theta_n, \xi_{n+1} )^\top \nabla \ell ( \theta_{n-1}, \xi_n ) }
&\leq \frac{1}{\gamma} [ f ( \theta_{n-1} ) - f ( \theta_n ) - \frac{c}{2} \| \theta_{n-1} - \theta_n \|^2 ] \\
&\quad + \frac{\beta}{\gamma} [ f ( \theta_{n-1} - \theta_{n-2} + \theta_n ) - f ( \theta_n ) - \frac{c}{2} \| \theta_{n-1} - \theta_{n-2} \|^2 ] \\
&\leq \frac{1}{\gamma} [ f ( \theta_{n-1} ) - f ( \theta_\star ) - \frac{c}{2} \| \theta_{n-1} - \theta_n \|^2 ] \\
&\quad+ \frac{\beta}{\gamma} [ f ( \theta_{n-1} - \theta_{n-2} + \theta_n ) - f ( \theta_\star )  - \frac{c}{2} \| \theta_{n-1} - \theta_{n-2} \|^2 ]
\end{align*}

Now extend the expectation over randomness in the trajectory, and apply the result from Theorem~\ref{thm:mom_convg_analysis}
and Lemma~\ref{lemma:lower_bound_thetadiff}.
\begin{align*}
&\leq \frac{1}{\gamma} \bigg[ \gamma M - \frac{c}{2} \gamma^2 \sigma_0^2 A_\beta \bigg] + \frac{\beta}{\gamma} \bigg[ \gamma M - \frac{c}{2} \gamma^2 \sigma_0^2 A_\beta \bigg] \\
&= ( 1 + \beta ) [ M - \frac{c}{2} \gamma \sigma_0^2 A_\beta ]
\end{align*}
\end{proof}

\emph{Remarks.}
$A_\beta = 1 / (1 + 2 \beta K + \beta^2)$ from  Lemma~\ref{lemma:lower_bound_thetadiff}.
$A_\beta$ is a monotonically decreasing function, where $A_{\beta | \beta=0} = 1$, and $A_{\beta | \beta=1} = 1 / (2 + 2K)$. Thus for the condition on the learning rate, higher $\beta$ causes lower $A_\beta$ which increases~the~condition~on~$\gamma^2$.\\

We now derive bounds similar to Theorem~\ref{thm:ip_exp_bd}
for the alternative test statistic constructed in Eq.~(\ref{eq:alt_test_stat}), and show that its expectation is highly sensitive to momentum.
This increased sensitivity makes it a poor choice to use in a convergence~diagnostic.\\

\begin{customthm}{\ref{thm:ip_opt}}
Suppose that Assumptions~\ref{assump:strcvx},~\ref{assump:Fbound}, and~\ref{assump:min_noise} hold. Then,
\begin{align*}
\Ex{ ( \nabla \ell ( \theta_n, \xi_{n+1}) + \beta & ( \theta_n - \theta_{n-1} ) )^\top ( \nabla \ell ( \theta_{n-1}, \xi_{n} ) + \beta ( \theta_{n-1} - \theta_{n-2} ) ) } \\
&< \left( \frac{1}{\gamma} + \frac{\beta}{\gamma} + 2 \beta + \beta^2 \right) \left[ \gamma M - \frac{c}{2} \gamma^2 \sigma_0^2 A_\beta \right]
+ \beta^3 \gamma M
\end{align*}
\end{customthm}

\begin{proof}
Decompose the terms in the inner product, apply expectation to both sides, and apply the strong convexity assumption.
\begin{align*}
\Ex{ ( \nabla \ell ( \theta_n, \xi_{n+1}) &+ \beta ( \theta_n - \theta_{n-1} ) )^\top ( \nabla \ell ( \theta_{n-1}, \xi_{n} ) + \beta ( \theta_{n-1} - \theta_{n-2} ) ) } \\
&= \Ex{\nabla \ell ( \theta_n, \xi_{n+1} )^\top \nabla \ell ( \theta_{n-1}, \xi_n )}
+ \Ex{\beta \nabla \ell ( \theta_n, \xi_{n+1} )^\top ( \theta_{n-1} - \theta_{n-2} )} \\
&\quad + \Ex{\beta \nabla \ell ( \theta_{n-1}, \xi_n )^\top ( \theta_n - \theta_{n-1} )}
+ \Ex{\beta^2 ( \theta_{n-1} - \theta_{n-2} )^\top ( \theta_{n} - \theta_{n-1} )} \\
&\leq
\frac{1}{\gamma} \left[ f( \theta_{n-1} ) - f( \theta_n ) - \frac{c}{2} \| \theta_{n-1} - \theta_n \|^2 \right] \\
&\quad + \frac{\beta}{\gamma} \left[ f( \theta_{n-1} - \theta_{n-2} + \theta_n ) - f( \theta_n ) - \frac{c}{2} \| \theta_{n-1} - \theta_{n-2} \|^2 \right] \\
&\quad + \beta \left[ f( \theta_{n-1} - \theta_{n-2} + \theta_n ) - f( \theta_n ) - \frac{c}{2} \| \theta_{n-1} - \theta_{n-2} \|^2 \right] \\
&\quad + \beta \left[ f( \theta_{n-1} ) - f( \theta_n ) - \frac{c}{2} \| \theta_{n-1} - \theta_n \|^2 \right] \\
&\quad + \beta^2 \left[ \beta \| \theta_{n-1} - \theta_{n-2} \|^2 + f( \theta_{n-2} ) - f( \theta_{n-1} ) - \frac{c}{2} \| \theta_{n-1} - \theta_{n-2} \|^2 \right]
\end{align*}
Where the last line uses $\theta_n - \theta_{n-1} = - \nabla \ell (\theta_{n-1}, \xi_{n} + \beta (\theta_{n-1} - \theta_{n-2} )$.
Take expectation with respect to the trajectory, and apply Theorem~\ref{thm:mom_convg_analysis},
Lemma~\ref{lemma:lower_bound_thetadiff}, and Lemma~\ref{lemma:MSE_SGDM_bound}.
\begin{align*}
&\leq \left( \frac{1}{\gamma} + \frac{\beta}{\gamma} + 2 \beta + \beta^2 \right) \left[ \gamma M - \frac{c}{2} \gamma^2 \sigma_0^2 A_\beta \right]
+ \beta^3\frac{\gamma M}{c}
\end{align*}
\end{proof}
\emph{Remarks.}
The expectation in Theorem~\ref{thm:ip_opt} is much more sensitive to the momentum parameter $\beta$, as seen in the additional $\beta$ terms.
The upper bound assumption on $\Ex{\| \theta_n - \theta_{n-1} \|^2}$ is not restrictive since we have a bound on $\Ex{f(\theta_n)-f(\theta_\star)}$ from Theorem~\ref{thm:mom_convg_analysis}.

\begin{customcor}{\ref{cor:ip_exp_neg}}
Consider SGDM in Eq.~(\ref{eq:sgdm}). If
the learning rate satisfies $\gamma > 2 M / c \sigma_0^2 A_\beta$, then,
\begin{equation*}
\mathbb{E} [ \nabla \ell ( \theta_n, \xi_{n+1} )^\top \nabla \ell ( \theta_{n-1}, \xi_n ) ]
< 0
\end{equation*}
as $n \rightarrow \infty$.
Thus the convergence diagnostic activates almost surely.
\end{customcor}
\begin{proof}
Combining Theorem~\ref{thm:ip_exp_bd} and the condition on the learning rate gives the desired bound.
\end{proof}

\noindent\textbf{Quadratic Loss Model.} \  Let's gain insight into the convergence diagnostic by first assuming quadratic loss $\ell ( \theta, y, x ) = \frac{1}{2} ( y - x^\top \theta_\star )^2$, $\nabla \ell ( \theta, y, x ) = - ( y - x^\top \theta) x$. Let $y = x^\top \theta_\star + \epsilon$, where $\epsilon$ are zero mean random variables $\Ex{ \epsilon | x } = 0$.
If we are able to initialize $\theta_0 = \theta_\star$, i.e. within the stationary region, then the update equations are:
\begin{align*}
\theta_1 &= \theta_\star + \gamma ( y_1 - x_1^\top \theta_\star ) x_1 \\
\theta_2 &= \theta_1 + \gamma ( y_2 - x_2^\top \theta_1 ) x_2 + \beta ( \theta_1 - \theta_\star ) \\
\theta_3 &= \theta_2 + \gamma ( y_3 - x_3^\top \theta_2 ) x_3 + \beta ( \theta_2 - \theta_1 )
\end{align*}
With the gradients reducing to:
\begin{align*}
\nabla \ell ( \theta_\star, y_1, x_1 ) &= ( y_1 - x_1^\top \theta_\star ) x_1 \\
&=  \epsilon_1 x_1\\
\nabla \ell ( \theta_1, y_2, x_2 ) &= ( y_2 - x_2^\top \theta_1 ) x_2 \\
&= ( y_2 - x_2^\top \theta_\star - \gamma ( y_1 - x_1^\top \theta_\star ) x_2^\top x_1 ) x_2 \\
&= ( \epsilon_2 - \gamma \epsilon_1 x_2^\top x_1 ) x_2 \\
\nabla \ell ( \theta_2, y_3, x_3 ) &= ( y_3 - x_3^\top \theta_2 ) x_3 \\
&= ( y_3 - x_3^\top \theta_1 - \gamma ( y_2 - x_2^\top \theta_1 ) x_3^\top x_2 - \beta x_3^\top ( \theta_1 - \theta_\star ) ) x_3 \\
&= ( y_3 - x_3^\top \theta_\star - \gamma ( y_1 - x_1^\top \theta_\star ) x_3^\top x_1 \\
&\quad - \gamma ( y_2 - x_2^\top \theta_\star - \gamma ( y_1 - x_1^\top \theta_\star ) x_2^\top x_1 ) x_3^\top x_2 \\
&\quad - \beta x_3^\top ( \gamma ( y_1 - x_1^\top \theta_\star ) x_1 ) ) x_3 \\
&= ( \epsilon_3 - \gamma \epsilon_1 x_3^\top x_1 - \gamma \epsilon_2 x_3^\top x_2 + \gamma^2 \epsilon_1 ( x_2^\top x_1 ) ( x_3^\top x_2 ) - \beta \gamma \epsilon_1 x_3^\top x_1 ) x_3
\end{align*}

The value of the expected test statistics is:
\begin{align*}
&\Ex{ \nabla \ell ( \theta_1, y_2, x_2 )^\top \nabla \ell ( \theta_2, y_3, x_3 ) }\\
=& \Ex{ - \gamma \epsilon_2^2 ( x_3^\top x_2 )^2 + \gamma^2 \epsilon_1^2 ( x_2^\top x_1 ) ( x_3^\top x_1 ) ( x_3^\top x_2 )
- \gamma^3 \epsilon_1^2 ( x_2^\top x_1 )^2 ( x_3^\top x_2 )^2 + \beta \gamma^2 \epsilon_1^2 ( x_2^\top x_1 ) ( x_3^\top x_1 ) ( x_3^\top x_2 ) } \\
=& - \gamma \Ex{ \epsilon_2^2 } \Ex{ ( x_3^\top x_2 )^2 } - \gamma^3 \Ex { \epsilon_1^2 } \Ex { ( x_2^\top x_1 )^2 ( x_3^\top x_2 )^2 }  + \gamma^2 ( 1 + \beta ) \Ex{ \epsilon_1^2 } \Ex{ ( x_2^\top x_1 ) ( x_3^\top x_1 ) ( x_3^\top x_2 ) } \\
\end{align*}

The following theorem shows how the test statistic evolves as the optimization procedure moves through~parameter~space.

\begin{customthm}{\ref{thm:quad_diag}}
Let the loss be quadratic, $\ell ( \theta ) = 1/2 ( y - x^\top \theta )^2$.
Let $x_n$, $x_{n+1}$ be two iid vectors from the distribution of $x$.
Let $A = \Ex{ ( x_n x_{n+1}^\top ) ( x_n^\top x_{n+1} ) }$, $B = \Ex{ ( x_{n} x_n^\top ) ( x_n^\top x_{n+1} )^2 }$, $\sigma_{quad}^2 = \Ex { \epsilon_n^2 }$, $d^2 = \Ex{ ( x_n^\top x_{n+1} )^2 }$.
Then,
\begin{align*}
&\Ex{ \nabla \ell ( \theta_n, \xi_{n+1} )^\top \nabla \ell ( \theta_{n-1}, \xi_n ) |  \theta_{n-1}, \theta_{n-2} }\\
=& ( \theta_{n-1} - \theta_\star )^\top ( A - \gamma B) ( \theta_{n-1} - \theta_\star )
- \gamma \sigma_{quad}^2 d^2  + ( \theta_{n-1} - \theta_\star )^\top ( \beta A ) ( \theta_{n-1} - \theta_{n-2} )
\end{align*}
\end{customthm}
\begin{proof}

Let $\nabla \ell(\theta_n) \equiv \nabla \ell( \theta_{n-1}, \xi_n)$.
The inner product is
\begin{align*}
&\nabla \ell ( \theta_n )^\top \nabla \ell ( \theta_{n-1} )\\
=&
\left( y_{n+1} - x_{n+1}^\top \theta_n \right)
\left( y_{n} - x_{n}^\top \theta_{n-1} \right)
x_{n}^\top x_{n+1} \\
=&
\biggl[
y_{n+1} - x_{n+1}^\top \theta_{n-1}
- \gamma ( y_{n} - x_{n}^\top \theta_{n-1} ) x_n^\top x_{n+1}
- \beta x_{n+1}^\top ( \theta_{n-1} - \theta_{n-2} )
\biggr]
\left( y_{n} - x_{n}^\top \theta_{n-1} \right)
x_{n}^\top x_{n+1}
\end{align*}
Distributing the second gradient term and trailing $x_n, x_{n+1}$ yields
\begin{align*}
&\nabla \ell ( \theta_n )^\top \nabla \ell ( \theta_{n-1} )\\
=&
\left( y_{n+1} - x_{n+1}^\top \theta_{n-1} \right) \left( y_{n} - x_{n}^\top \theta_{n-1} \right) x_{n}^\top x_{n+1}
 - \gamma \left( y_{n} - x_{n}^\top \theta_{n-1} \right)^2 \left( x_{n}^\top x_{n+1} \right)^2 \\
&- \beta \left( y_{n} - x_{n}^\top \theta_{n-1} \right) [ x_{n+1}^\top ( \theta_{n-1} - \theta_{n-2} ) ] \left( x_{n}^\top x_{n+1} \right)
\end{align*}
And now substitute for $\epsilon$
\begin{align*}
&\nabla \ell ( \theta_n )^\top \nabla \ell ( \theta_{n-1} )\\
=&
\biggl[ x_{n+1}^\top ( \theta_\star - \theta_{n-1} ) + \epsilon_{n+1} \biggr]
\biggl[ x_{n}^\top ( \theta_\star - \theta_{n-1} ) + \epsilon_{n} \biggr] x_{n}^\top x_{n+1} \\
& \quad - \gamma \biggl[ x_{n}^\top ( \theta_\star - \theta_{n-1} ) + \epsilon_{n} \biggr]^2
\left( x_{n}^\top x_{n+1} \right)^2 \\
& \quad - \beta \biggl[ x_{n}^\top ( \theta_\star - \theta_{n-1} ) + \epsilon_{n} \biggr]
\biggl[ x_{n+1}^\top ( \theta_{n-1} - \theta_{n-2} ) \biggr]
 \left( x_{n}^\top x_{n+1} \right)
 \end{align*}
 Expand terms
 \begin{align*}
 &\nabla \ell ( \theta_n )^\top \nabla \ell ( \theta_{n-1} )\\
 =&
 ( \theta_{n-1} - \theta_\star )^\top \biggl[ ( x_n x_{n+1}^\top ) ( x_n^\top x_{n+1} ) \biggr] ( \theta_{n-1} - \theta_\star )
 + \epsilon_n W^{(1)} + \epsilon_{n+1} W^{(2)} + \epsilon_n \epsilon_{n+1} W^{(3)} \\
 & \quad - ( \theta_{n-1} - \theta_\star )^\top \biggl[ \gamma ( x_{n} x_{n}^\top ) ( x_n^\top x_{n+1} )^2 \biggr] ( \theta_{n-1} - \theta_\star )
 - \epsilon_n W^{(4)} - \gamma \epsilon_n^2 ( x_n^\top x_{n+1} )^2 \\
 & \quad + ( \theta_{n-1} - \theta_\star )^\top \biggl[ \beta ( x_n x_{n+1}^\top ) ( x_n^\top x_{n+1} ) \biggr] ( \theta_{n-1} - \theta_{n-2} )
 - \epsilon_n W^{(5)}
\end{align*}
Take expectation with respect to $n-1, n-2$.
Conditioning with respect to two terms is necessary to capture the momentum.2
The non square epsilon terms are eliminated because they are mean zero.\\
\end{proof}

\section{Proofs for Section~\ref{sec:distribution_IP}}
We first give some technical Lemmas necessary to prove the main variance bound results.

\begin{lemma}
\label{lemma:MSE_SGDM_bound}
Suppose that Assumptions~\ref{assump:strcvx} and~\ref{assump:Fbound} hold for some positive $M' > 0$. Let $M = 8 M'$.
Then we can bound the L2 distance iterates,

\begin{equation*}
\Ex{ \| \theta_n - \theta_{n-1} \|^2 } \leq \frac{ \gamma M }{ c }
\end{equation*}
\end{lemma}

\begin{proof}
From Theorem~\ref{thm:mom_convg_analysis}~\cite{Report:Yang_arXiv16}, we can see that for large enough $n$, this is a valid condition.
Under the strong convexity assumption, we can bound
\begin{align*}
\frac{c}{2} \| \theta_n - \theta_\star \|^2 &\leq f ( \theta_n ) - f ( \theta_\star ) - \nabla f ( \theta_\star )^\top ( \theta_n - \theta_\star ) \\
&= f ( \theta_n ) - f ( \theta_\star )
\end{align*}
Applying expectation  to both sides,
\begin{align*}
\frac{c}{2} \ \Ex{ \| \theta_n - \theta_\star \|^2 } &\leq \Ex{ f ( \theta_n ) - f ( \theta_\star ) }
\end{align*}

We now use the triangle inequality  to obtain the desired bound.
\begin{align*}
\| \theta_n - \theta_{n-1} \| &\leq \| \theta_n - \theta_\star \|  + \| \theta_{n-1} - \theta_\star \| \\
\| \theta_n - \theta_{n-1} \|^2 &\leq \| \theta_n - \theta_\star \|^2  + \| \theta_{n-1} - \theta_\star \|^2 + 2 \| \theta_n - \theta_\star \| \| \theta_{n-1} - \theta_\star \|
\end{align*}
Apply expectation to both sides,
\begin{align*}
\Ex{ \| \theta_n - \theta_{n-1} \|^2 }
&\leq \Ex{ \| \theta_n - \theta_\star \|^2 } + \Ex{ \| \theta_{n-1} - \theta_\star \|^2 } + 2 \Ex{ \| \theta_n - \theta_\star \| } \Ex{ \| \theta_{n-1} - \theta_\star \| } \\
&\leq \Ex{ \| \theta_n - \theta_\star \|^2 } + \Ex{ \| \theta_{n-1} - \theta_\star \|^2 } + 2 \sqrt{ \Ex{ \| \theta_n - \theta_\star \|^2 } \Ex{ \| \theta_{n-1} - \theta_\star \|^2 } }\\
&\leq 8 \frac{ \gamma M' }{ c }
\end{align*}

In the stationary phase the variance of the stochastic gradient dominates and no more progress is made towards $\theta_\star$, so then $\| \theta_n - \theta_\star \|^2$ and $\| \theta_{n-1} - \theta_\star \|^2$ are independent in the stationary phase.
The second inequality is by Jensen's.
\end{proof}

\newpage

\begin{lemma}
\label{lemma:exp_ip_sq_lower}
Consider the SGDM procedure in Eq.~(\ref{eq:sgdm}).
Suppose that Assumptions~\ref{assump:Lsmooth} and~\ref{assump:Fbound} hold.

Then we can bound
\begin{equation*}
\mathbb{E}\biggl[ \biggl( \nabla \ell ( \theta_n, \xi_{n+1} )^\top \nabla \ell ( \theta_{n-1}, \xi_n ) \biggr)^2 \biggr]
\geq ( 1 + 2 \beta + \beta^2 ) \biggl[ M' - \frac{L}{2} \gamma \sigma_0^2 A_\beta \biggr]^2
\end{equation*}
\end{lemma}

\begin{proof}
Apply expectation with respect to $\xi_{n+1}$, Jensen's inequality, and the decomposition in Eq.~(\ref{eq:ip_breakdown}).
\begin{align*}
&\Ex{ ( \nabla \ell ( \theta_n, \xi_{n+1} )^\top \nabla \ell ( \theta_{n-1}, \xi_n ) )^2 } \\
\geq& \Ex{ ( \nabla \ell ( \theta_n, \xi_{n+1} )^\top \nabla \ell ( \theta_{n-1}, \xi_n ) ) }^2 \\
=& \bigg( \Ex{ \frac{1}{\gamma} \nabla \ell ( \theta_n, \xi_{n+1} )^\top ( \theta_{n-1} - \theta_n ) }
+ \Ex{ \frac{\beta}{\gamma}  \nabla \ell ( \theta_n, \xi_{n+1} )^\top ( \theta_{n-1} - \theta_{n-2} ) } \bigg)^2
\end{align*}

Define $f(\theta) \equiv \Ex{ \ell (\theta, \xi) }$. Now apply the Lipschitz smoothness condition.
\begin{align*}
&\Ex{ ( \nabla \ell ( \theta_n, \xi_{n+1} )^\top \nabla \ell ( \theta_{n-1}, \xi_n ) )^2 } \\
\geq& \biggl( \frac{1}{\gamma} [ \ell ( \theta_{n-1} ) - \ell ( \theta_n ) - \frac{L}{2} \| \theta_{n-1} - \theta_n \|^2 ]  + \frac{\beta}{\gamma} [ \ell ( \theta_{n-1} - \theta_{n-2} + \theta_n ) - \ell ( \theta_n ) - \frac{L}{2} \| \theta_{n-1} - \theta_{n-2} \|^2 ] \biggr)^2 \\
=& \frac{1}{\gamma^2} \bigg[ \ell ( \theta_{n-1} ) - \ell ( \theta_n ) - \frac{L}{2} \| \theta_{n-1} - \theta_n \|^2 \bigg]^2
 + 2 \frac{\beta}{\gamma^2} \bigg[ \ell ( \theta_{n-1} ) - \ell ( \theta_n ) - \frac{L}{2} \| \theta_{n-1} - \theta_n \|^2 \bigg] \\
&\quad \quad \times\bigg[ \ell ( \theta_{n-1} - \theta_{n-2} + \theta_n ) - \ell ( \theta_{n} ) - \frac{L}{2} \| \theta_{n-1} - \theta_{n-2} \|^2 \bigg] \\
&+ \frac{\beta^2}{\gamma^2} \bigg[ \ell ( \theta_{n-1} - \theta_{n-2} + \theta_n ) - \ell ( \theta_{n} ) - \frac{L}{2} \| \theta_{n-1} - \theta_{n-2} \|^2 \bigg]^2
\end{align*}

For brevity of notation define $\Delta \ell _{n-1, n} \equiv \ell ( \theta_{n-1} ) - \ell ( \theta_n )$. Then
\begin{align*}
&\Ex{ ( \nabla \ell ( \theta_n, \xi_{n+1} )^\top \nabla \ell ( \theta_{n-1}, \xi_n ) )^2 } \\
=& \frac{1}{\gamma^2} \biggl[ \Delta \ell _{n-1, n}^2 - 2 \Delta \ell _{n-1, n} \frac{L}{2} \| \theta_{n-1} - \theta_n \|^2 + \frac{L^2}{4} \| \theta_{n-1} - \theta_n \|^4 \biggr] \\
&+ 2 \frac{\beta}{\gamma^2} \biggl[ \Delta \ell _{n-1, n} \Delta \ell_{n-1n-2n, n} - \Delta \ell _{n-1n-2n, n} \frac{L}{2} \| \theta_{n-1} - \theta_n \|^2 \\
&- \Delta \ell_{n-1,n} \frac{L}{2} \| \theta_{n-1} - \theta_{n-2} \|^2  + \frac{L^2}{4} \| \theta_{n-1} - \theta_n \|^2 \| \theta_{n-1} - \theta_{n-2} \|^2 \biggr] \\
& + \frac{\beta^2}{\gamma^2} \biggl[ \Delta \ell_{n-1n-2n,n}^2 - 2 \Delta \ell_{n-1n-2n,n} \frac{L}{2} \| \theta_{n-1} - \theta_{n-2} \|^2 + \frac{L^2}{4} \| \theta_{n-1} - \theta_{n-2} \|^4 \biggr]
\end{align*}

Now apply expectation with respect to the trajectory, Theorem~\ref{thm:mom_convg_analysis}, and Lemma~\ref{lemma:lower_bound_thetadiff}. Let $A_\beta = (\frac{1}{1 + 2 K \beta + \beta^2})$.
\begin{align*}
&\Ex{ ( \nabla \ell ( \theta_n, \xi_{n+1} )^\top \nabla \ell ( \theta_{n-1}, \xi_n ) )^2 } \\
\geq& \frac{1}{\gamma^2} \biggl[ \gamma^2 M'^2 - M' L \gamma^3 \sigma_0^2 A_\beta + \frac{L^2}{4} \gamma^4 \sigma_0^4 A_\beta^2 \biggr] \\
&\quad + 2 \frac{\beta}{\gamma^2} \biggl[ \gamma^2 M'^2 - M' L \gamma^3 \sigma_0^2 A_\beta + \frac{L^2}{4} \gamma^4 \sigma_0^4 A_\beta^2 \biggr] \\
&\quad + \frac{\beta^2}{\gamma^2} \biggl[ \gamma^2 M'^2 - M' L \gamma^3 \sigma_0^2 A_\beta + \frac{L^2}{4} \gamma^4 \sigma_0^4 A_\beta^2 \biggr] \\
=& \biggl( \frac{1 + 2 \beta + \beta^2}{\gamma^2} \biggr) \biggl[ \gamma M' - \frac{L}{2} \gamma^2 \sigma_0^2 A_\beta \biggr]^2
\end{align*}

To bound $\Ex{ \Delta_{n-1,n}^2 }$, $\Ex{ \Delta_{n-1,n}^2 } \geq \Ex{ \Delta_{n-1,n} }^2$ by Jensen's inequality, and $\Ex{ \Delta_{n-1,n} } \geq \Ex{ \Delta_{\star, n} } \geq - \gamma M'$ through our assumption.
To bound $\Ex{ \Delta_{n-1, n}  \| \theta_{n-1} - \theta_n \|^2 }$, we first use $\Delta_{n-1,n} \geq \Delta_{\star,n}$.
We then use the Cauchy-Schwarz and then Jensen's inequality to bound $\Ex{ \Delta_{n-1, n}  \| \theta_{n-1} - \theta_n \|^2 } \geq \sqrt{ \Ex{ \Delta_{\star, n}^2 } } \sqrt{ \Ex{ \| \theta_{n-1} - \theta_n \|^4 } } \geq \Ex{ \Delta_{\star, n} } \Ex{ \| \theta_{n-1} - \theta_n \|^2 } \geq - \gamma M' \gamma^2 \sigma_0^2 A_\beta$,
and use the Theorem~ \ref{thm:mom_convg_analysis}
bound and Lemma~\ref{lemma:lower_bound_thetadiff}.
We bound $\Ex{ \| \theta_{n-1} - \theta_n \|^4 } \geq \Ex{ \| \theta_{n-1} - \theta_n \|^2 }^2$ with Jensen's inequality, and can then use Lemma~\ref{lemma:lower_bound_thetadiff}.
$\Delta_{n-1,n}$ and $\Delta_{n-1n-2n, n}$ are independent in the stationary region, and thus $\Ex{ \Delta_{n-1n-2n, n} \Delta_{n}} \geq \gamma^2 M'^2$.
To bound $ \frac{L^2}{4} \| \theta_{n-1} - \theta_n \|^2 \| \theta_{n-1} - \theta_{n-2}\|^2$, we first note that the two terms are positively correlated due to momentum.
For two random variables $X, Y$ with positive covariance, $\Ex {X Y } = Cov (X , Y ) + \Ex{X} \Ex{Y} \geq \Ex{X} \Ex{Y}$.
Thus, $\Ex{ \| \theta_{n-1} - \theta_n \|^2 \| \theta_{n-1} - \theta_{n-2}\|^2} \geq \Ex{ \| \theta_{n-1} - \theta_n \|^2 } \Ex{ \| \theta_{n-1} - \theta_{n-2}\|^2}$.
\end{proof}

\begin{lemma}
\label{lemma:exp_ip_lower}
Consider the SGDM procedure in Eq.~(\ref{eq:sgdm}).
Suppose that Assumptions~\ref{assump:Lsmooth} and~\ref{assump:Fbound} hold. Then we can lower bound
\begin{equation*}
\mathbb{E}[ \nabla \ell ( \theta_n, \xi_{n+1} )^\top \nabla \ell ( \theta_{n-1}, \xi_n ) ] \geq
- (1 + \beta ) M' (1 + 4L/c )
\end{equation*}
\end{lemma}

\begin{proof}
First apply the decomposition in Eq.~(\ref{eq:ip_breakdown}),
\begin{equation*}
\nabla \ell ( \theta_n, \xi_{n+1} )^\top \nabla \ell ( \theta_{n-1}, \xi_n )
= \frac{1}{\gamma} \nabla \ell ( \theta_n, \xi_{n+1} )^\top ( \theta_{n-1} - \theta_n )
+ \frac{\beta}{\gamma}  \nabla \ell ( \theta_n, \xi_{n+1} )^\top ( \theta_{n-1} - \theta_{n-2} )
\end{equation*}
Then apply expectation with respect to $\xi_{n+1}$ and use the Lipschitz smoothness assumption to create the desired inequality
\begin{align*}
\Ex{ \nabla \ell ( \theta_n, \xi_{n+1} )^\top \nabla \ell ( \theta_{n-1}, \xi_n ) }
&\geq \frac{1}{\gamma} [ f ( \theta_{n-1} ) - f ( \theta_n ) - \frac{L}{2} \| \theta_{n-1} - \theta_n \|^2 ] \\
&\quad + \frac{\beta}{\gamma} [ f ( \theta_{n-1} - \theta_{n-2} + \theta_n ) - f ( \theta_n ) - \frac{L}{2} \| \theta_{n-1} - \theta_{n-2} \|^2 ]
\end{align*}
Now apply expectation with respect to the trajectory, Theorem~\ref{thm:mom_convg_analysis}, and Lemma~\ref{lemma:MSE_SGDM_bound}.
\begin{align*}
&\Ex{ \nabla \ell ( \theta_n, \xi_{n+1} )^\top \nabla \ell ( \theta_{n-1}, \xi_n ) }\\
\geq& \frac{1}{\gamma} [ - \gamma M' - L \gamma 8 M' / c ]
+ \frac{\beta}{\gamma} [ - \gamma M' - L \gamma 8 M' / c ] \\
=& - (1 + \beta ) M' (1 + 8L/c )
\end{align*}
\end{proof}

\begin{customthm}{\ref{thm:ip_var_bound}}
Consider the SGDM procedure in Eq.~(\ref{eq:sgdm}). Suppose that Assumptions~\ref{assump:strcvx},~\ref{assump:Lsmooth},~\ref{assump:Fbound},~\ref{assump:min_noise}, and~\ref{assump:scaling} hold. Define $IP = \nabla \ell(\theta_n, \xi_{n+1})^\top \nabla \ell(\theta_{n-1}, \xi_n)$.
Then,
\begin{align*}
\frac{Var [ IP ]}{\mathbb{E} [ IP ]^2}
&\geq \frac{ ( M' - L \gamma \sigma_0^2 A_\beta )^2 }{M'^2 ( 1 + 8 L / c )^2 }
- 1
\end{align*}
\end{customthm}

\begin{proof}
\begin{align*}
\frac{Var [ IP ]}{\mathbb{E} [ IP ]^2} &= \frac{\mathbb{E} [ IP^2 ] - \mathbb{E} [ IP ]^2}{\mathbb{E} [ IP ]^2}
= \frac{\mathbb{E} [ IP^2 ]}{\mathbb{E} [ IP ]^2}  - 1
\end{align*}
We have a lower bound on $\Ex{IP^2}$ from Lemma~\ref{lemma:exp_ip_sq_lower}, and an upper bound on $\Ex{IP}^2$ from Lemma~\ref{lemma:exp_ip_lower}.
We use Lemma~\ref{lemma:exp_ip_lower} and not the bound from Theorem~\ref{thm:ip_exp_bd}
because $|(1+\beta)M'(1+8L/c)| \geq |(1+\beta)[M' - \frac{c}{2} \gamma \sigma_0^2 A_\beta]|$.
\begin{align*}
\frac{Var [ IP ]}{\mathbb{E} [ IP ]^2} &= \frac{\mathbb{E} [ IP^2 ]}{\mathbb{E} [ IP ]^2}  - 1
\geq \frac{ ( 1 + \beta )^2 ( M' - \frac{L}{2} \gamma \sigma_0^2 A_\beta )^2 } { (1 + \beta )^2 M'^2 ( 1 + 8 L / c )^2 }
- 1
\end{align*}
\end{proof}

\begin{customcor}{\ref{cor:set_gamma_threshold}}
Consider the SGDM procedure in Eq.~(\ref{eq:sgdm}).
Fix a scaling factor $\lambda > 2$.
Set the learning rate $\gamma = 2 t M' / L \sigma_0^2 A_\beta$ with $t  \geq 1 + \sqrt{ \lambda } ( 1 + 4 L / c )$. Then the lower bound in Theorem~\ref{thm:ip_var_bound}
is bounded

\begin{equation*}
\frac{Var [ IP ]}{\mathbb{E} [ IP ]^2}  \geq \lambda - 1
\end{equation*}
\end{customcor}

\begin{proof}
Setting $\gamma = 2 t M' / L \sigma_0^2 A_\beta$, we see that
\begin{equation*}
\frac{Var [ IP ]}{\mathbb{E} [ IP ]^2}
\geq \frac{ ( 1 - t )^2} { ( 1 + 8 L / c )^2 } - 1.
\end{equation*}

The lower bound is obtained by solving the inequality,
\begin{align*}
( 1 - t )^2 &\geq \lambda ( 1 + 8 L / c )^2 \\
t^2 - 2t + 1 - \lambda ( 1 + 8 L / c )^2 &\geq 0 \\
\end{align*}

We solve the quadratic inequality.
Because it is convex, we find
\begin{align*}
t &\geq \biggl| \frac{2 \pm \sqrt{ 4 - 4( 1 - \lambda ( 1 + 8 L / c )^2 ) }}{2} \biggr| \\
&= | 1 \pm ( 1 + 8 L / c ) \sqrt{ \lambda } |
\end{align*}
Thus $t = 1 + \sqrt{\lambda}(1 + 8 L / c)$ because we need $\gamma > 0$.
\end{proof}

\end{document}